\documentclass[aos]{imsart}

\RequirePackage{amsthm,amsmath,amsfonts,amssymb}
\RequirePackage[numbers,sort&compress]{natbib}
\RequirePackage[colorlinks,citecolor=blue,urlcolor=blue]{hyperref}
\RequirePackage{graphicx}

\usepackage{anyfontsize} 
\usepackage{newtxtext} 
\usepackage{mismath} 
\usepackage{enumitem} 
\usepackage{booktabs} 
\usepackage{cleveref} 

\newcommand{\Alpha}{A} 
\newcommand{\defeq}{\stackrel{\text{def}}{=}} 
\newcommand{\varhyphen}[1]{{\operatorname{\mathit{#1}}}} 
\setlist[description]{leftmargin=\parindent,labelindent=\parindent,rightmargin=\parindent} 

\startlocaldefs

\theoremstyle{plain}
\newtheorem{axiom}{Axiom}
\newtheorem{claim}[axiom]{Claim}

\theoremstyle{definition}

\endlocaldefs

\begin{document}

\begin{frontmatter}
\title{Common TF–IDF variants arise as key components in the test statistic of a penalized likelihood-ratio test for word burstiness}

\begin{aug}
\author[A]{\fnms{Zeyad}~\snm{Ahmed}\ead[label=e1]{zaawad@upei.ca}\orcid{0009-0003-0781-2968}},
\author[A]{\fnms{Paul}~\snm{Sheridan}\ead[label=e2]{paul.sheridan.stats@gmail.com}\orcid{0000-0002-5484-1951}},
\author[A]{\fnms{Michael}~\snm{McIsaac}\orcid{0000-0002-5273-5450}},
\and
\author[B,C]{\fnms{Aitazaz~A.}~\snm{Farooque}\orcid{0000-0002-5353-6752}}
\address[A]{School of Mathematical and Computational Sciences,
University of Prince Edward Island \printead[presep={,\ }]{e1,e2}}

\address[B]{Canadian Centre for Climate Change and Adaptation, 
University of Prince Edward Island}

\address[C]{Faculty of Sustainable Design Engineering, University of Prince Edward Island}

\end{aug}

\begin{abstract}
TF–IDF is a classical formula that is widely used for identifying important terms within documents. We show that TF–IDF-like scores arise naturally from the test statistic of a penalized likelihood-ratio test setup capturing word burstiness (also known as word over-dispersion). In our framework, the alternative hypothesis captures word burstiness by modeling a collection of documents according to a family of beta-binomial distributions with a gamma penalty term on the precision parameter. In contrast, the null hypothesis assumes that words are binomially distributed in collection documents, a modeling approach that fails to account for  word burstiness. We find that a term-weighting scheme given rise to by this test statistic performs comparably to TF–IDF on document classification tasks. This paper provides insights into TF–IDF from a statistical perspective and underscores the potential of hypothesis testing frameworks for advancing term-weighting scheme development.
\end{abstract}


\end{frontmatter}


\section{Introduction}
Term Frequency-Inverse Document Frequency (TF–IDF) is a classical term-weighting scheme that has become widespread in information retrieval and text analysis. It quantifies the extent to which a given term of interest stands out in any one given document out of many. Specifically, TF–IDF assigns each term in a document a weight equal to its raw frequency in that document multiplied by the negative logarithm of the proportion of documents in which the term occurs. The higher a term's TF–IDF score in a document, the better it may be considered to characterize the document's content.

While TF–IDF was initially conceived as a heuristic, considerable progress has been made over the years in establishing theoretical foundations for its use, though most justifications are oriented toward the information retrieval community~\citep[see][]{Joachims1997, Hiemstra2000, Amati2002, Aizawa2003, Elkan2005, Roelleke2013, Havrlant2017}. In recent work, however, Sheridan~et~al.~\cite{Sheridan2026} took first steps to explicitly motivate TF–IDF from a statistical perspective. But justifications aimed at the statistical community remain nascent. In this paper we build on this emerging line of research to fill that gap by showing how common TF–IDF variants arise as key components in the test statistic of a penalized likelihood-ratio (PLR) testing framework comparing two statistical models of language, one of which explicitly captures a fundamental property of natural language called word burstiness.

In the context of a collection of textual documents, word burstiness, also known as word over-dispersion, refers to the tendency of certain words to appear in concentrated repetitions within few documents rather than being independently scattered across all collection documents~\cite[see][p. 284]{Irvine2017}. For instance, in a collection of news articles, such terms as ``vote'' and ``candidate'' are liable to appear repeatedly within the subset of articles pertaining to elections. This phenomenon violates the assumptions of the multinomial statistical model of language that is traditionally used in bag-of-words representations of document collections.

Madsen~et~al.~\cite{Madsen2005} addressed this shortcoming by modeling word burstiness using a family of Dirichlet-multinomial distributions to represent a collection of documents. Elkan~\cite{Elkan2005} subsequently showed that TF–IDF-like components emerge from the Fisher kernel of this model. This marked a significant theoretical advance, establishing a formal connection between TF–IDF and a statistical language model sophisticated enough to account for the word burstiness phenomenon. In the process, Elkan elegantly resolved what we dub the \emph{Event Space Problem of Information Retrieval}, as articulated by Robertson~\cite{Robertson2004, Robertson2005}, which encompasses the difficulty of treating TF–IDF within a probabilistic framework due to a mismatch in the event spaces of its TF and IDF components. Namely, the TF component is defined over the sample space of a single document, while the IDF component is defined over the sample space of documents in the entire collection.

Sheridan~et~al.~\cite{Sheridan2026} showed how TF–IDF emerges from Fisher's exact test of statistical significance under restrictive assumptions that enforce a contrived form of word burstiness. They did not, however, model this phenomenon outright. Here, we pick up where Sheridan~et~al.~\cite{Sheridan2026} left off by modeling word burstiness within a hypothesis testing framework. In particular, we propose a PLR test statistic for quantifying word burstiness that compares a binomial null model with a penalized beta-binomial alternative. The binomial null model, by definition, fails to account for word burstiness. In contrast, the beta-binomial model, which is a special case of the Dirichlet-multinomial model, accounts for word burstiness by accommodating for variability in term occurrences across documents. The likelihood ratio is penalized through the incorporation of a gamma-distributed penalty term that breaks a ridge present in the unpenalized alternative model. Notably, in part by exploiting the very same resolution to the Event Space Problem as in Elkan~\cite{Elkan2005}, we show that two common TF–IDF variants, BTF–IDF and TF–ICF, emerge as key components of the resulting PLR test statistic.

We emphasize that our PLR construction does not constitute a test in the classical statistical sense, as we do not advance a procedure for computing corresponding $p$-values. Rather, our focus lies squarely on establishing a theoretical link between the resulting test statistic and TF–IDF-like term-weightings.

In a simulated example, we found PLR test statistic scores to be strongly correlated with the sum of TF–IDF weights across documents ($r=0.9269$). Terms with higher TF–IDF scores tended to have larger values of the test statistic. Buoyed by this finding, we propose a term-weighting scheme that is motivated by the PLR test statistic. In real data document classification tasks, we found that the proposed term-weighting scheme performed comparably to TF–IDF. On the \textit{20~Newsgroups} dataset~\cite{Lang1995}, the proposed term-weighting scheme achieved an overall accuracy of~0.80, versus~0.81 for TF–IDF. Likewise, on the \textit{R8} dataset~\cite{CardosoCachopo2007}, a subset of the \textit{Reuters-21578} text categorization collection~\cite{Lewis1987}, the proposed term-weighting scheme performed comparably to TF–IDF, with both achieving~0.94 overall accuracy. Although TF–IDF did noticeably outperform the proposed scheme on highly imbalanced minority classes.
 
The primary contribution of this paper is to extend emerging theoretical justifications for TF–IDF that are grounded in statistical hypothesis testing~\citep[see][]{Sheridan2024, Sheridan2026}. In particular, we leveraged the Event Space Problem solution of Elkan~\cite{Elkan2005} to demonstrate that two common TF–IDF variants emerge naturally as components of a PLR test statistic for word burstiness, providing a direct connection between TF–IDF-like term-weighting schemes and statistical hypothesis testing. This contribution is distinct from that of Elkan~\cite{Elkan2005}, who shows that the Fisher kernel associated with a Dirichlet-multinomial model of text yields TF–like and IDF-like components. In contrast, our framework yields TF–IDF variants as key components of a PLR test statistic. In addition, we demonstrate that the term-weighting scheme given rise to by this test statistic performs comparably to TF–IDF in text classification settings. It is our aspiration that this work will inspire future investigations into the theoretical foundations of term-weighting schemes within the hypothesis testing paradigm.

The remainder of the paper is organized as six sections. \Cref{sec:background} provides preliminary background information. In~\Cref{subsec:bow-model} we introduce the notation for the bag-of-words model we use throughout the paper to model collections of documents. \Cref{subsec:tfidf-family} discusses term-weighting schemes. Three members of the TF–IDF family are defined, namely, TF–IDF, BTF–IDF, and TF–ICF. \Cref{subsec:dm-model} introduces the Dirichlet-multinomial statistical model of language. Outlined is the Dirichlet-multinomial framework for modeling document collections and the associated term-generation process. \Cref{subsec:bb-model} introduces the beta-binomial statistical model of language as a special case of the Dirichlet-multinomial model. Furthermore, we demonstrate the effect of changing the precision parameter on the distribution of term counts within documents. In~\Cref{subsec:event-space} we discuss the Event Space Problem of Information Retrieval. \Cref{sec:related-work} reviews prior work on modeling word burstiness using the Dirichlet-multinomial model, and justifications of TF–IDF rooted in word burstiness-aware models. \Cref{subsec:modeling-burstiness} outlines the key developments in the Dirichlet-multinomial model of language. \Cref{subsec:justifications-tfidf} presents existing justifications for TF–IDF using models that capture the word burstiness phenomenon. \Cref{sec:gbb-model} introduces the penalized beta-binomial model and its approximation. \Cref{subsec:bb-approximation} derives an approximation to the beta-binomial model under certain regularity conditions. \Cref{subsec:gamma-penalty} introduces a gamma distributed penalty term on the precision parameter to regularize the objective. \Cref{subsec:gbb-model-approximation} introduces the penalized beta-binomial language model that will be used in the rest of the paper. \Cref{sec:gbb-mles} derives the penalized maximum likelihood estimates for the approximated penalized beta-binomial model we developed. An interpretation of the results is conducted. \Cref{sec:gbb-param-estimation} illustrates the ridge that forms from the beta-binomial model approximation and how introducing a penalty term on the precision parameter breaks that ridge. \Cref{sec:likelihood-ratio-test} introduces a PLR test statistic for quantifying word burstiness and shows that TF–IDF variants emerge from it. In \Cref{subsec:lrt-formulation}, we outline the likelihood-ratio test framework and define the null and alternative hypotheses. \Cref{subsec:evaluation-of-lrt} derives the PLR test statistic. In particular, we show that two common TF–IDF variants emerge as key terms in the PLR test statistic. \Cref{sec:numerical-examples} compares the test statistic with TF–IDF using synthetic and real data. \Cref{subsec:numerical-synthetic} compares the test statistic with the sum of TF–IDF scores and show that there is a strong positive correlation between the two. \Cref{subsec:numerical-real} and \Cref{subsec:numerical-real-ii} demonstrate the validity of the proposed model's assumptions, and use a term-weighting scheme motivated by the test statistic and show that it performs comparably to TF–IDF in document classification tasks on the \textit{20~Newsgroups} and \textit{R8} datasets, respectively. \Cref{sec:discussion} concludes the paper with a summary of our main findings and a discussion of some potential directions for future work.

\section{Background}
\label{sec:background}
This section establishes the notation used throughout the paper, defines TF–IDF and related term-weighting schemes, identifies a key complication in developing a probabilistic interpretation of TF–IDF, and presents the beta-binomial statistical language model for document collections.

\subsection{Bag-of-words modeling framework} 
\label{subsec:bow-model}
Table~\ref{tab:bag-of-words-model} summarizes the bag-of-words model notation we use for representing a collection of $d>0$ documents (denoted $d_1,...,d_d$), each composed from a vocabulary of $m>0$ distinct terms (denoted $t_1,...,t_m$). Documents are treated as multisets of their constituent terms, meaning that term order is disregarded. Each document $d_j$ ($1 \leq j \leq d$) is associated with a term frequency vector $\boldsymbol{n}_j = (n_{ij})_{i=1}^m$, where $n_{ij} \geq 0$ indicates the count of term $t_i$ in document~$d_j$. In the paper, we use $d_1,...,d_d$ as generic symbols denoting the documents themselves and~$\boldsymbol{n}_j$ for the term frequency vector representation of $d_j$.

\begin{table*}[!th]
\caption{Notation for the bag-of-words model representation of a collection of documents.}
\centering
\resizebox{\textwidth}{!}{
\begin{tabular}{p{100pt}p{100pt}p{260pt}}
Symbol\thinspace /\thinspace Expr. & Name (Abbr.) & Description \\
\hline
$\mathcal{C}$ & Collection & Generic symbol for a collection of $d>0$ documents composed of $m~>~0$ distinct terms. \\
$d>0$ & Collection size & Number of documents in the collection. \\
$m>0$ & Vocabulary size & Number of distinct terms out of which collection documents are composed. \\
$t_i$ ($1\leq i\leq m$) & $i$'th term & Generic symbol for the $i$'th term in the vocabulary. \\
$d_j$ ($1\leq j\leq d$) & $j$'th document & Generic symbol for the $j$'th document in the collection. \\
$b_{ij}$ & Binary term frequency (BTF) & Indicator function defined as $1$ if the $i$'th term occurs in the $j$'th document, and $0$ otherwise. \\
$n_{ij}$ & Term frequency (TF) & Number of occurrences of the $i$'th term in the $j$'th document. \\
$b_i=\sum_{j=1}^d b_{ij}$ & Document frequency & Number of documents containing the $i$'th term. \\
$b_i/d$ & Document proportion & Proportion of documents containing the $i$'th term. \\
$\log(d/b_i)$ & Inverse document frequency (IDF) & Natural logarithm of the inverse document proportion $d/b_i$. \\
$n_i = \sum_{j=1}^d n_{ij}$ & Collection frequency & Number of occurrences of the $i$'th term in the collection. \\
$n_i/n$ & Collection proportion & Proportion of occurrences of the $i$'th term in the collection. \\
$\log(n/n_i)$ & Inverse collection frequency (ICF) & Natural logarithm of the inverse collection proportion $n/n_i$. \\
$r_i$ & N/A & Shorthand for $n_i - b_i + 1$. \\
$n_j = \sum_{i=1}^m n_{ij}$ & Document length & Total number of terms in the $j$'th document. \\
$\boldsymbol{n}_j = (n_{ij})_{i=1}^m$ & N/A & Term frequency vector for the $j$'th document. \\
$n = \sum_{j=1}^d n_j$ & Total collection frequency & Total number of terms in the collection. \\
$\mathcal{C}_i$ & Target-complement document collection & Representation of a collection $\mathcal{C}$ of $d>0$ documents with the vocabulary collapsed into two terms: the target $t_i$ and its complement $t_{\neg i}$ \\
\hline
\end{tabular}
}
\label{tab:bag-of-words-model}
\end{table*}

\subsection{TF–IDF and related term-weighting schemes}
\label{subsec:tfidf-family}
A \emph{term-weighting scheme} is any function that assigns a nonnegative, real-valued weight to each term-document pair within a collection. The weight of a term in a document is meant to quantify the term's importance in characterizing the document's content relative to the collection as a whole. Term-weighting schemes have traditionally been a fundamental component of text analysis methods used for performing such tasks as text classification, document retrieval, and topic modeling.

TF–IDF is a long-established term-weighting scheme with a storied history~\citep[see][]{SparckJones2004}. Formulated by Salton and Yang~\cite{Salton1973}, TF–IDF is defined as the product of two factors
\begin{equation} \label{eq:tfidf}
\varhyphen{TF–IDF}(i,j) \defeq n_{ij} \log(d / b_i),
\end{equation}
where $b_i$ is the document frequency as defined in Table~\ref{tab:bag-of-words-model}. Note that $\log x$, here and elsewhere in paper, denotes the natural logarithm of a positive real number, $x$. The term frequency factor $n_{ij}$ quantifies the importance of the $i$'th term in the $j$'th document. The inverse document frequency term $\log(d / b_i)$ downweights the $i$'th term according to how frequently it appears across the documents in the collection. Terms with high TF–IDF scores appear frequently in a given document but infrequently across the collection, making such terms effective for distinguishing that document from others in the collection.

TF–IDF is famously the template for numerous variants~\citep[see][]{Alshehri2023, Tang2024}, two of which are particularly relevant to the present work. The first is a simple variation on TF–IDF known as \emph{binary term frequency–inverse document frequency} (BTF–IDF), and is defined as
\begin{equation} \label{eq:btfidf}
\varhyphen{BTF–IDF}(i,j) \defeq b_{ij} \log(d / b_i).
\end{equation}
The only change from TF–IDF is that the term frequency $n_{ij}$ is swapped out for the indicator variable $b_{ij}$, which, as we have seen in Table~\ref{tab:bag-of-words-model}, accounts for the presence or absence of the $i$'th term in the $j$'th document. BTF–IDF is found among the various term-weighting schemes evaluated by Salton and Buckley~\cite{Salton1988} in their highly influential study of document retrieval methods. The second is the \emph{term frequency–inverse collection frequency} (TF–ICF), introduced by Kwok~\cite{Kwok1990}, and defined as 
\begin{equation} \label{eq:tficf}
\varhyphen{TF–ICF}(i,j) \defeq n_{ij} \log(n / n_i).
\end{equation}
The ICF factor, also traceable to Kwok~\cite{Kwok1990}, is an IDF cousin that characterizes term rarity across the entire collection without taking into account how terms are distributed across individual documents.

In recent years, word embeddings have largely supplanted traditional term-weighting schemes in text analysis by providing dense, low-dimensional representations that capture semantic similarity and contextual usage~\citep[see][]{Islam2023}. Word embeddings, which are trained on large textual corpora, have become foundational in state-of-the-art neural network models for such standard tasks such as text classification, translation, and question answering. However, TF–IDF family term-weighting schemes remain far from obsolete. They continue to offer interpretable, efficient, and effective baselines. This is particularly the case in resource-constrained settings, classical information retrieval, and scenarios where transparency and simplicity are desired over model complexity~\citep[see][]{Graff2025}.

\subsection{Modeling word burstiness with the Dirichlet-multinomial statistical model of language}
\label{subsec:dm-model}
The Dirichlet-multinomial distribution is a family of discrete distributions that arise naturally in Bayesian modeling of overdispersed multinomial processes. In the context of language modeling, Madsen~et~al.~\cite{Madsen2005} proposed the Dirichlet-multinomial statistical model of language which captures the word burstiness phenomenon.

Consider a collection, $\mathcal{C}$, of $d$ documents
composed from a vocabulary of $m$ terms. Under the Dirichlet-multinomial model, each document $d_j$ ($1 \leq j \leq d$) is modeled according to a Dirichlet-multinomial distribution with $n_j$ trials, and vector of positive \emph{concentration} parameters $\boldsymbol{\alpha} = (\alpha_1, \ldots, \alpha_m$) that is held in common by all documents. In this model, the multinomial parameter vector of term proportions for each document $d_j$ is drawn as $\boldsymbol{\Theta}_j \sim \text{Dir}(\boldsymbol{\alpha})$, and realized as~$\boldsymbol{\theta}_j = (\theta_{1j}, \ldots, \theta_{mj})$. Given~$\boldsymbol{\theta}_j$, term frequencies are generated according to $\boldsymbol{N}_j \sim~\text{Mult}(n_j, \boldsymbol{\theta}_j)$ with $\boldsymbol{N}_j$ realized as $\boldsymbol{n}_j=(n_{1j}, \ldots, n_{mj})$. Estimation theory for this model was pioneered by Minka~\cite{Minka2000}. The resulting distribution over the term count vector~$\boldsymbol{N}_j$ is
\begin{eqnarray}
    \Pr(\boldsymbol{n}_j \mid \boldsymbol{\alpha}) 
    & = & \int_{\boldsymbol{\theta}_j} \text{Mult}(\boldsymbol{n}_j \mid \boldsymbol{\theta}_j) \, \text{Dir}(\boldsymbol{\theta}_j \mid \boldsymbol{\alpha}) \,\, d\boldsymbol{\theta}_j \\
    & = & \frac{n_j!}{\prod_{i = 1}^{m} n_{ij}!} \cdot \frac{\Gamma(\alpha_0)}{\Gamma(n_j + \alpha_0)} \cdot \prod_{i = 1}^{m} \frac{\Gamma(n_{ij} + \alpha_i)}{\Gamma(\alpha_i)},
    \label{eq:dcm}
\end{eqnarray}
where the sum $\alpha_0=\sum_{i=1}^{m}{\alpha_i}$ denotes the \emph{precision} parameter and $\Gamma(\cdot)$ the gamma function.
For a derivation and further discussion we refer the reader to Minka~\cite{Minka2000} and Madsen~et~al.~\cite{Madsen2005}. It is this model that is commonly employed in language modeling applications~\citep[see][]{Madsen2005, Elkan2006, Cummins2015, Cummins2017}.

The Dirichlet-multinomial model allows for variability in $\boldsymbol{\theta_j}$ via the Dirichlet prior. The precision parameter $\boldsymbol{\alpha}$ controls the shape and concentration of the prior distribution over term proportions. In particular, the precision parameter $\alpha_0$ governs how sharply the sampled multinomial parameters $\boldsymbol{\theta}_j$ are distributed. For a term $t_i$, the ratio~$\alpha_i / \alpha_0$ represents the expectation of the term's corresponding $\theta_{ij}$ value. When $\alpha_0$ is large, $\theta_{ij}$ becomes concentrated around this expected value, leading to less variability in $\theta_{ij}$ across documents. In the limit as $\alpha_0 \xrightarrow{}\infty$, the distribution converges to the classical multinomial model with variability in $\theta_{ij}$ eliminated across documents, ignoring word burstiness often present in typical textual documents. Conversely, when $\alpha_0$ is small, the sampled $\theta_{ij}$ values exhibit high variability, which captures word burstiness.

\subsection{The beta-binomial statistical model of language} 
\label{subsec:bb-model}
The beta-binomial distribution is a family of discrete probability distributions that arise naturally in Bayesian modeling of overdispersed binomial processes. It has a wide range of applications and has been marshaled in such fields as ecology~\citep[see][]{Harrison2015}, epidemiology~\citep[see][]{Bakbergenuly2017}, and actuarial science~\citep[see][]{Panjer1992} where success counts across multiple trials exhibit greater variability than afforded by a standard binomial model.

When modeling a collection of textual documents, the beta-binomial distribution serves as a natural choice for capturing the word burstiness phenomenon in the framework where each document is represented in terms of a target term $t_i$ and its complement $t_{\neg i}$. We call the set of such documents the \emph{target-complement collection}, denoted by $\mathcal{C}_i$.

In this modeling framework, we collapse the vocabulary to two terms by isolating a single \emph{target term}~$t_i$ and pooling all remaining terms into a single \emph{complement term}~$t_{\neg i}$. Under the beta-binomial model, each document $d_j$ ($1 \leq j \leq d$) is modeled according to a beta-binomial distribution with $n_j$ trials and \emph{concentration} parameters $\alpha_i > 0$ and $\alpha_{\neg i} > 0$, denoted in vector form as $\boldsymbol{\alpha} = (\alpha_i, \alpha_{\neg i})$, that are held in common by all documents. In this model, the binomial success probability $\theta_{ij}$ for the target term $t_i$ in document $d_j$ is drawn from a beta prior, $\Theta_{ij} \sim \text{Beta}(\boldsymbol{\alpha})$, where $\Theta_{ij}$ is a random variable with realization $\theta_{ij}$, $\Theta_{\neg ij}=1-\Theta_{ij}$ is the corresponding random variable for the complement term with realization $\theta_{\neg ij}$. We denote the random vector of term proportions as $\boldsymbol{\Theta}_j = (\Theta_{ij}, \Theta_{\neg ij})$ and its realization as $\boldsymbol{\theta}_j = (\theta_{ij}, \theta_{\neg ij})$. The beta-binomial distribution for any given document, $d_j$, is a compound distribution in which the binomial distribution success probability $\theta_{ij}$ is drawn from a beta distribution $\Theta_{ij} \sim \text{Beta}(\boldsymbol{\alpha})$. Given~$\theta_{ij}$, term frequencies are generated according to $N_{ij} \sim~\text{Bin}(n_j, \theta_{ij})$ where $\boldsymbol{n}_j=(n_{ij}, n_{\neg ij})$ is a realization of the term frequencies for document $d_j$ under the \emph{target-complement} representation. Since $n_{\neg {ij}} = n_j - n_{ij}$, the term frequency $n_{ij}$ alone needs to be generated by the model. The resulting probability mass function is given by
\begin{eqnarray}
    \Pr(n_{ij} \mid \boldsymbol{\alpha}) 
    &=& \int_{\theta_{ij}} \text{Bin}(n_{ij} \mid \theta_{ij})  \, \text{Beta}(\theta_{ij} \mid \boldsymbol{\alpha}) \,\, d\theta_{ij} \\
    &=& \binom{n_j}{n_{ij}} \cdot \frac{B(n_{ij} + \alpha_i, n_j - n_{ij} + \alpha_{\neg i})}{B(\alpha_i, \alpha_{\neg i})} \\
    &=& \frac{n_j!}{n_{ij}! \, n_{\neg ij}!} \cdot  \frac{\Gamma(\alpha_{0i})}{\Gamma(n_j + \alpha_{0i})} \cdot \frac{\Gamma(n_{ij} + \alpha_i)}{\Gamma(\alpha_i)} \cdot \frac{\Gamma(n_{\neg ij} + \alpha_{\neg i})}{\Gamma(\alpha_{\neg i})}, \label{eq:bcb-pmf}
\end{eqnarray}
where the sum $\alpha_{0i}=\alpha_i + \alpha_{\neg i}$ is the \emph{precision} parameter, $B(\cdot,\cdot)$ denotes the beta function, and $\Gamma(\cdot)$ the gamma function. Note that in standard formulations of the beta-binomial distribution, the concentration parameters are typically denoted as $\alpha$ and $\beta$. Here, however, we use the symbols $\alpha_i$ and $\alpha_{\neg i}$ to emphasize the connection with its Dirichlet-multinomial model generalization. Also note that the \emph{precision} parameter $\alpha_{0i}$ here depends on how we collapse the vocabulary according to the target term $t_i$, unlike the Dirichlet-multinomial model's precision parameter $\alpha_0$.

Similar to the Dirichlet-multiomial model, the parameters $\alpha_i$ and $\alpha_{\neg i}$ jointly govern the shape and concentration of the distribution. When $\alpha_{0i}$ is large,~$\theta_{ij}$ is concentrated around its expected value~$\alpha_i/\alpha_{0i}$, resulting in behavior similar to the binomial model. When~$\alpha_{0i}$ is small, there is more variability in $\theta_{ij}$ across documents, thus modeling burstiness in word usage. In the limiting case where the variability in~$\theta_{ij}$ is eliminated, the beta-binomial distribution converges to the binomial distribution.

Figure~\ref{fig:bcb-viz} demonstrates the behavior of the beta-binomial distribution for different combinations of~$\alpha_i$ and $\alpha_{\neg i}$. Each sample point represents a term count distribution for a document. For a fixed ratio $\alpha_i/\alpha_{0i}$, we observe that when~$\alpha_{0i}$ is small, the points exhibit high variability and are widely scattered. As~$\alpha_{0i}$ increases, the points cluster more tightly around the expected value, reflecting reduced variability.

\begin{figure}[!ht]
\includegraphics[width=\linewidth]{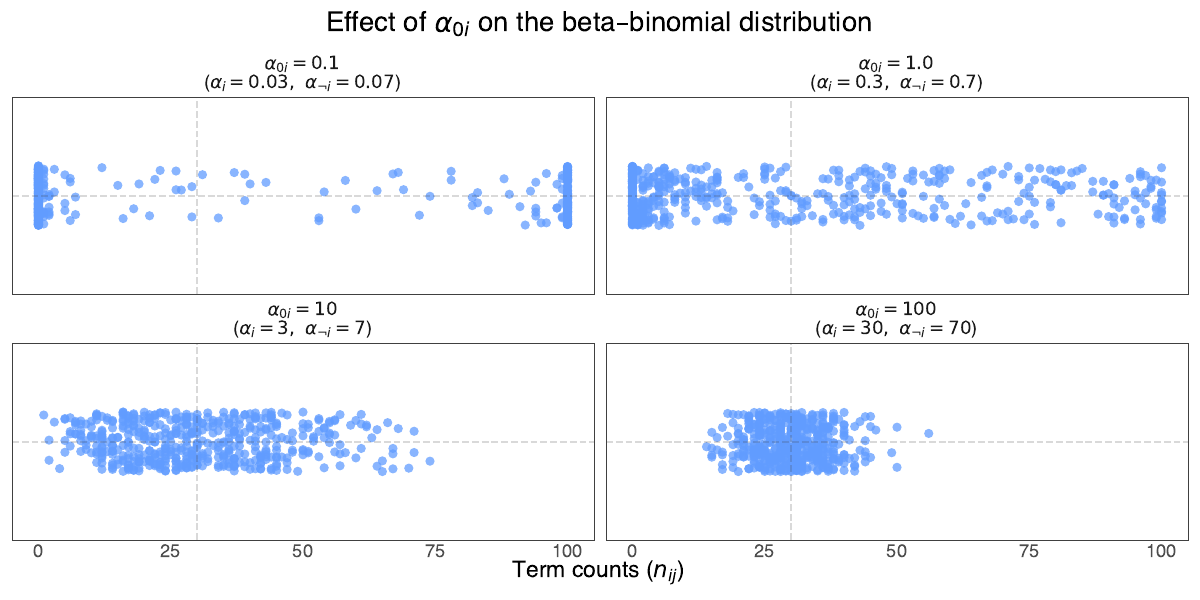}
\caption{Beta-binomial samples of target term $t_i$ counts in documents with fixed $\alpha_i/\alpha_{0i} = 0.3$ and increasing precision~$\alpha_{0i}~\in~\{0.1, 1.0, 10, 100\}$. Each point represents the count of $t_i$ in a single document of length $n_j=100$.}
\label{fig:bcb-viz}
\end{figure}

\subsection{The Event Space Problem of Information Retrieval}
\label{subsec:event-space}
Robertson~\cite{Robertson2004, Robertson2005} articulated a fundamental obstacle that hampers attempts to connect TF–IDF with probability theory, and by extension, with information theory and statistical language modeling. Robertson described this complication as ``an event space problem''~\cite[see][p. 505]{Robertson2004}. For ease of exposition, we refer to it as the Event Space Problem of Information Retrieval, or the Event Space Problem for short. In general, it concerns the difficulty of defining compatible event spaces for probability distributions over queries, documents, and relevance. In this work, we focus on a specific aspect of the problem: the challenge of integrating the term event space (i.e., the space of term occurrences within a document over which the TF component of TF–IDF is defined) with the document event space (i.e., the space of documents within a collection over which the IDF component of TF–IDF is defined) into a coherent probabilistic modeling framework. We will see that the estimation theory developed for the Dirichlet-multinomial model by Minka~\cite{Minka2000}, as applied in an underappreciated work by Elkan~\cite{Elkan2005}, offers an elegant resolution to this problem by naturally incorporating IDF into the structure of statistical language modeling. We adopt this same approach within a beta-binomial modeling framework.

\section{Related work}
\label{sec:related-work}
Previous work has exploited the Dirichlet-multinomial model for representing document collections on account of its ability to capture word burstiness. In parallel, other research has sought to justify the use of TF–IDF by establishing connections between it and statistical language models that explicitly account for the word burstiness phenomenon.

\subsection{Modeling word burstiness in text analysis}
\label{subsec:modeling-burstiness}
In an early application, Madsen~et~al.~\cite{Madsen2005} showed that the Dirichlet-multinomial model outperforms its multinomial model counterpart when it comes to classification accuracy on document classification tasks. Elkan~\cite{Elkan2006} then introduced an exponential-family approximation to the Dirichlet-multinomial (EDCM) that reformulates the model in a more analytically tractable form for analysis and training. The approximation makes maximum-likelihood training much faster than for the Dirichlet-multinomial. Elkan further applied EDCM in document clustering algorithms, achieving strong empirical performance. Cummins~\cite{Cummins2015} proposed the Smoothed P\'{o}lya Urn Document language model (SPUD) by smoothing a document specific Dirichlet-multinomial model with a background Dirichlet-mutinomial model estimated from the whole collection, and demonstrated improvements over standard multinomial language models in document retrieval tasks. Building on this, Cummins~\cite{Cummins2017} proposed a generalized multivariate P\'{o}lya process, encompassing models like the multinomial and the Dirichlet-multinomial as special cases via a replacement matrix. This replacement matrix defines a particular type of document language model depending on the diagonal entries of the matrix. In this framework, the Dirichlet-multinomial is equivalent to the replacement matrix being the identity matrix assuming all terms are equally bursty, whereas the multinomial model is equivalent to the zero matrix. Cummins mainly focused on the generalized matrix where the diagonal entries are either tuned or estimated based on the data using a generalized SPUD language model.

\subsection{Justifications for TF–IDF rooted in word burstiness}
\label{subsec:justifications-tfidf}
Using this modeling background, several papers connected TF–IDF with word burstiness aware models. Elkan~\cite{Elkan2005} derived the Fisher kernel of the Dirichlet-multinomial model and showed that it contains terms reminiscent of TF and IDF. Sunehag~\cite{Sunehag2007} modeled word burstiness using a two-stage presence/abundance model that is equivalent to a generalized P\'{o}lya urn model, which is closely related to the Dirichlet-multinomial and its EDCM approximation. Sunehag showed that TF–IDF follows naturally from a cross-entropy coding argument in this framework. Recently, Sheridan~et~al.~\cite{Sheridan2026} derived TF–IDF from Fisher's exact test under restrictive conditions that enforce an extreme form of word burstiness, leaving an opening for a less ad hoc, more statistics based justification under less restrictive assumptions.

\section{The penalized beta-binomial language model for word burstiness and its approximation}
\label{sec:gbb-model}
In this section we introduce an extension of the beta-binomial language model, which we dub the penalized beta-binomial language model, that incorporates a gamma distributed penalty term on the beta-binomial model precision parameter. We devise an approximation of this model that is tailored to the scenario in which the target term $t_i$ exhibits burstiness while the complement term $t_{\neg i}$ does not. We moreover derive penalized maximum likelihood estimates for the parameters of the approximate model and demonstrate that the inclusion of the gamma penalty term is critical for breaking a ridge in the nonpenalized likelihood function.

\subsection{An approximation to the beta-binomial language model}
\label{subsec:bb-approximation}
In this subsection, we secure an approximation to the beta-binomial language model that captures the behavior of bursty terms in document collections under specific assumptions.
\begin{claim}
\label{claim:cl1}
Under the beta-binomial language model, the probability of a document, $d_j$, within a target-complement collection, $\mathcal{C}_i$, is closely approximated by
\begin{equation} \label{eq:bcb-approx-1}
    \Pr(n_{ij} \mid \boldsymbol{\alpha}) \approx \frac{n_j!}{n_{\neg ij}!} \left(\frac{\alpha_i}{n_{ij}} \right)^{b_{ij}} \frac{\alpha_{\neg i}^{n_{\neg ij}}}{\alpha_{0i}^{n_j}}
\end{equation}
provided that the target term $t_i$ has a relatively small concentration parameter (i.e., $\alpha_i \ll 1$), and the complement term $t_{\neg i}$ has a relatively large concentration parameter (i.e., $\alpha_{\neg i} \gg 1$) and furthermore occurs in each document (i.e., $b_{\neg i} = d$).
\end{claim}

\begin{proof}[Derivation]
To begin, we rewrite the beta-binomial model probability mass function of Eq.~\eqref{eq:bcb-pmf} as
\begin{equation} \label{eq:bcb-exact-with-indicator}
    \Pr(n_{ij} \mid \boldsymbol{\alpha}) = \frac{n_j!}{e^{b_{ij} \log n_{ij}!} \cdot n_{\neg ij}!} \cdot \frac{\Gamma(\alpha_{0i})}{\Gamma(n_j + \alpha_{0i})} \cdot e^{b_{ij} \log \frac{\Gamma(n_{ij} + \alpha_i)}{\Gamma(\alpha_i)}} \cdot \frac{\Gamma(n_{\neg ij} + \alpha_{\neg i})}{\Gamma(\alpha_{\neg i})}
\end{equation}
where $b_{ij} \in \{0,1\}$ is an indicator accounting for the target term $t_i$ appearing in document $d_j$. The corresponding indicator $b_{\neg{ij}}$ of the complement term $t_{\neg i}$ does not appear in the equation since we assume $b_{\neg i} = d$ (i.e., $b_{\neg{ij}} = 1 \forall j$) which is a trivial condition since we do not consider cases where a document only consists of the target term $t_i$. The form of Eq.~\eqref{eq:bcb-exact-with-indicator} allows for the connection between the sample space of a single document and the sample space of documents in the collection when computing the log-likelihood as we will show later. This form was inspired by Elkan~\cite{Elkan2006} representing the Dirichlet-multinomial in its exponential family representation.

Next, we approximate each of the three gamma function quotients found to occur in Eq.~\eqref{eq:bcb-exact-with-indicator}. Consider first the target term quotient $\Gamma(n_{ij} + \alpha_i)/\Gamma(\alpha_i)$. Given that $0 < \alpha_i \ll 1$, we can use the approximation $\Gamma(\alpha_i) = \frac{1}{\alpha_i} - \gamma + \bigO(\alpha_i)$~\cite[\href{https://dlmf.nist.gov/5.7\#E4}{(5.7.4)}]{NIST2025}, where $\gamma$ denotes the Euler–Mascheroni constant. Likewise since $\alpha_i$ is small, the numerator $\Gamma(n_{ij} + \alpha_i)$ can be approximated by $\Gamma(n_{ij}) + \bigO(\alpha_i)$ as detailed in Appendix~\ref{appn:gamma-approx}. This becomes $(n_{ij}-1)! + \bigO(\alpha_i)$ since $n_{ij}$ is confined to take on positive integer values. The complement term quotient $\Gamma(n_{\neg ij} + \alpha_{\neg i})/\Gamma(\alpha_{\neg i})$ can be approximated as $n_{\neg ij}^{\alpha_{\neg i}} + \bigO(1/\alpha_{\neg i})$ since $\alpha_{\neg i}$ is large~\cite[\href{https://dlmf.nist.gov/5.11\#E12}{(5.11.12)}]{NIST2025}. By the same reasoning, the precision parameter quotient $\Gamma(\alpha_{0i})/\Gamma(n_j + \alpha_{0i})$ can be approximated by the reciprocal of $n_{j}^{\alpha_{0i}} + \bigO(1/\alpha_{0i})$.

Using these approximations, along with basic properties of exponents and logarithms, we isolate the main contributing terms in Eq.~\eqref{eq:bcb-exact-with-indicator} and simplify to obtain the exact relation
\begin{equation} \label{eq:bcb-special-case-exact}
    \Pr(n_{ij} \mid \boldsymbol{\alpha}) = \frac{n_j!}{n_{\neg ij}!} \left(\left(\frac{1}{n_{ij}} + \bigO(\alpha_i)\right) \frac{\alpha_i}{\left(1-\gamma \alpha_i + \bigO(\alpha_i^2)\right)} \right)^{b_{ij}} \frac{\alpha_{\neg i}^{n_{\neg ij}}\left(1 + \bigO(\alpha_{\neg i}^{-1})\right)}{\alpha_{0i}^{n_j}\left(1+\bigO(\alpha_{0i}^{-1})\right)}.
\end{equation}
Here, the exponential terms are simplified using the identity $e^{\log x} = x$, where $x$ is a positive real number.

As a final step, we simplify Eq.~\eqref{eq:bcb-special-case-exact} by dropping any big~$\bigO$ terms and recognizing that $1-\gamma \alpha_i = 1 - \bigO(\alpha_i) \approx 1$ since $\alpha_i \ll 1$, yielding the desired approximation of Eq.~\eqref{eq:bcb-approx-1}. This completes the derivation.
\end{proof}

\subsection{A gamma distributed penalty term for the beta-binomial model precision parameter}
\label{subsec:gamma-penalty}
We will see shortly that leaving the beta-binomial language model precision parameter $\alpha_{0i}$ unconstrained makes it difficult to maximize the objective. To address this problem, we introduce a gamma distributed penalty term for modeling $\alpha_{0i}$.

Take $\Alpha_{0i}$ to be a random variable whose realizations correspond to instantiations of~$\alpha_{0i}$. We assume $\Alpha_{0i} \sim \mathcal{G}(\mu,\sigma^2)$ follows a gamma distribution, parameterized by its mean $\E[\Alpha_{0i}] = \mu$ and variance $\V[\Alpha_{0i}] = \sigma^2$, with probability density function
\begin{equation} \label{eq:gamma-penalty}
    f(\alpha_{0i}; \mu, \sigma^2) = 
    \frac{\left( \frac{\mu}{\sigma^2} \right)^{\mu^2 / \sigma^2}}{ \Gamma\left( \frac{\mu^2}{\sigma^2} \right) }
    \alpha_{0i}^{\frac{\mu^2}{\sigma^2} - 1}
    e^{-\frac{\mu}{\sigma^2} \alpha_{0i}}.
\end{equation}
The gamma distribution constitutes a natural choice for modeling such scale-related quantities as precision. Its support on the positive real line aligns with the domain of $\alpha_{0i}$, and its two-parameter form allows for sufficient flexibility represent uncertainty spanning several orders of magnitude. In addition, the parameters $\mu$ and $\sigma^2$ are amenable to tuning. In subsequent numerical examples, we choose $\mu$ to be the mean document length so that the precision $\alpha_{0i}$ is centered around it. The choice of $\mu$ is motivated by empirical observations by Elkan~\cite{Elkan2005} which indicate that both the precision parameter and the average document length commonly take on values in the hundreds in the Dirichlet-multinomial framework.

\subsection{Penalized beta-binomial language model definition and approximation}
\label{subsec:gbb-model-approximation}
We define the penalized beta-binomial language model for a target-complement document collection $\mathcal{C}_i$, according to 
\begin{equation} \label{eq:gbb-pmf}
    L_{pen}(\mathcal{C}_i) = \prod_{j=1}^d \Pr(n_{ij} \mid \boldsymbol{\alpha}) \cdot f(\alpha_{0i}; \mu, \sigma^2),
\end{equation}
where $\Pr(n_{ij} \mid \boldsymbol{\alpha})$, as set out in Eq.~\eqref{eq:bcb-pmf}, denotes the beta-binomial probability of document~$d_j$. We use the notation $L_{pen}(\mathcal{C}_i)$ instead of $\Pr(\mathcal{C}_i)$ to emphasize that the inclusion of the penalty term means the resulting quantity is technically not a probability distribution. The density $f(\alpha_{0i}; \mu, \sigma^2)$ corresponds to the gamma penalty term of Eq.~\eqref{eq:gamma-penalty}. It is important to emphasize that the gamma penalty term is defined at the collection level, rather than at the level of individual documents. The formula 
\begin{equation} \label{eq:gbb-approx-pmf}
    L_{pen}(\mathcal{C}_i) \approx \prod_{j=1}^d \left(\frac{n_j!}{n_{\neg ij}!} \left(\frac{\alpha_i}{n_{ij}} \right)^{b_{ij}} \frac{\alpha_{\neg i}^{n_{\neg ij}}}{\alpha_{0i}^{n_j}}\right) \cdot f(\alpha_{0i}; \mu, \sigma^2)
\end{equation}
defines the penalized beta-binomial model approximation that will serve as our primary focus from this point forward. The expression inside the product is the approximate beta-binomial document-level probability mass function as specified in Eq.~\eqref{eq:bcb-approx-1}.

\subsection{Penalized maximum likelihood estimation} \label{sec:gbb-mles}
It is now time to derive the maximum likelihood estimates for the parameters of our approximation to the penalized beta-binomial language model. 

\begin{claim}
The maximum likelihood estimates for the parameters $\alpha_i$ and $\alpha_{\neg i}$ in the approximate penalized beta-binomial language model of a document collection $\mathcal{C}_i$ are given by
\begin{equation}
    \label{eq:bcb-penalized-mles}
    \hat{\alpha}_i = \hat{\alpha}_{0i}\frac{b_i}{b_i + n_{\neg i}} , \quad \text{and} \quad \hat{\alpha}_{\neg i} = \hat{\alpha}_{0i}\frac{n_{\neg i} }{b_i + n_{\neg i}},
\end{equation}
where $\hat{\alpha}_{0i} = \frac{\sigma^2}{\mu}\left(b_i - n_i - 1 + \mu^2 / \sigma^2 \right)$.
\label{claim:cl2}
\end{claim}

\begin{proof}[Derivation]
From Eq.~\eqref{eq:bcb-approx-1} the approximate log-likelihood of document $d_j$ ($1 \leq j \leq d$) is
\begin{equation}
    \label{eq:bb-approx-ll}
    \ell(\boldsymbol{\alpha} \mid d_j) = \log n_j! - \log n_{\neg{ij}}! + b_{ij} \log \alpha_i - b_{ij} \log n_{ij} + n_{\neg{ij}} \log \alpha_{\neg i} - n_j \log \alpha_{0i}.
\end{equation}

The log-likelihood of $\boldsymbol{\alpha}$ for the collection $\mathcal{C}_i$, including the penalty term on the precision quantity $\Alpha_{0i}$, is
\begin{equation}
\label{eq:gbb-log-likelihood}
    \ell(\boldsymbol{\alpha} \mid \mathcal{C}_i) \defeq \ell(\mathcal{C}_i) = \sum_{j=1}^{d}\ell(\boldsymbol{\alpha} \mid d_j) + \log f(\alpha_{0i}).
\end{equation}
Taking partial derivatives of $\ell(\mathcal{C}_i)$ with respect to $\alpha_i$ and $\alpha_{\neg i}$ we get
\begin{equation}
    \frac{\partial l(\mathcal{C}_i)}{\partial \alpha_i} = \frac{b_i}{\alpha_i} - \frac{n}{\alpha_{0i}} + \frac{\mu^2 - \sigma^2}{\sigma^2 \alpha_{0i}} - \frac{\mu}{\sigma^2}, 
    \quad \text{and} \quad
    \frac{\partial l(\mathcal{C}_i)}{\partial \alpha_{\neg i}} = \frac{n_{\neg i}}{\alpha_{\neg i}} - \frac{n}{\alpha_{0i}} + \frac{\mu^2 - \sigma^2}{\sigma^2 \alpha_{0i}} - \frac{\mu}{\sigma^2}.
\end{equation}
Setting each expression to zero and solving for $\hat\alpha_i$ and $\hat\alpha_{\neg i}$, we obtain 
\begin{equation}
    \hat\alpha_i = \frac{b_i \sigma^2 \hat\alpha_{0i}}{\mu \hat\alpha_{0i} + n \sigma^2 - \mu^2 + \sigma^2}, 
    \quad \text{and} \quad
    \hat\alpha_{\neg i} = \frac{n_{\neg i} \sigma^2 \hat\alpha_{0i}}{\mu \hat\alpha_{0i} + n \sigma^2 - \mu^2 + \sigma^2}.
    \label{eq:mle-with-a0}
\end{equation}
We find $\hat\alpha_{0i}$ by taking the sum $\hat\alpha_i + \hat\alpha_{\neg i}$. Solving for $\hat\alpha_{0i}$ gives
\begin{equation}
    \label{eq:a0-mle-pen}
    \hat{\alpha}_{0i} = \frac{\sigma^2}{\mu}\left(b_i - n_i - 1 + \mu^2 / \sigma^2 \right).
\end{equation}
Substituting this value of $\hat{\alpha}_{0i}$ back into Eq.~\eqref{eq:mle-with-a0} yields 
\begin{equation}
    \hat{\alpha}_i = \frac{b_i \sigma^2 \left(b_i - n_i - 1 + \mu^2/\sigma^2 \right)}{\mu \left(b_i + n_{\neg i} \right)}, \quad \text{and} \quad
    \hat{\alpha}_{\neg i} = \frac{n_{\neg i} \sigma^2 \left(b_i - n_i - 1 + \mu^2/\sigma^2 \right)}{\mu \left(b_i + n_{\neg i} \right)},
\end{equation}
which leads to the expressions for $\hat{\alpha}_i$ and $\hat{\alpha}_{\neg i}$ in Eq.~\eqref{eq:bcb-penalized-mles} as was set out to show.
\end{proof}

Observe that the penalized maximum likelihood estimates from Eq.~\eqref{eq:bcb-penalized-mles} show $\hat\alpha_i$ to be roughly proportional to $b_i$, and $\hat\alpha_{\neg i}$ to $n_{\neg i}$. Given that the model assumes $\alpha_i \ll 1$, the fact that~$\hat{\alpha}_i$ is tied to the number of documents in which the term appears (i.e., $b_i$), rather that to its total frequency~$n_i$, suggests that document-level presence provides more information about $\alpha_i$ than does its aggregate term count at the collection level. By contrast, for the the complement term~$t_{\neg i}$, where the model assumes a corresponding $\alpha_{\neg i} \gg 1$, we find~$\hat\alpha_{\neg i}$ to be tied to the total number of its occurrences across the collection, $n_{\neg i}$, suggesting that $\alpha_{\neg i}$ is predominantly recoverable from its collection-level term count alone.

The penalty on $\Alpha_{0i}$ plays a critical role in yielding a unique penalized maximizer. In its absence, the penalized maximum likelihood estimates of $\alpha_i$ and $\alpha_{\neg i}$ reduce to
\begin{equation}
    \hat\alpha_i = \hat\alpha_{0i} \frac{b_i}{n}
    \quad \text{and} \quad
    \hat\alpha_{\neg i} = \hat\alpha_{0i} \frac{n_{\neg i}}{n},
\end{equation}
respectively. Since by definition $\alpha_{0i} = \alpha_i + \alpha_{\neg i}$, adding the above estimates for $\hat\alpha_i$ and $\hat\alpha_{\neg i}$ yields
\begin{equation}
    \hat\alpha_{0i} = \hat\alpha_{0i} \frac{b_i}{n} + \hat\alpha_{0i} \frac{n_{\neg i}}{n},
\end{equation}
which leaves $\hat\alpha_{0i}$ as a free parameter. In this case, any $\hat\alpha_{0i}$ would satisfy the equation. Consequently, the unpenalized approximate model of Eq.~\eqref{eq:bcb-approx-1} is unidentifiable with respect to the precision parameter $\alpha_{0i}$. This further implies the pathological condition $n = b_i + n_{\neg i}$, which holds if, and only, if the target term occurs exactly once in every document in which it occurs at all (i.e., $b_i = n_i$). 

We note that $\hat{\alpha}_{0i}$ in Eq.~\eqref{eq:a0-mle-pen} must be positive to yield valid parameter estimates. This requires $ \mu^2/\sigma^2 > b_i - n_i - 1$. In typical collections, this condition is generally satisfied when $\mu$ and $\sigma^2$ are chosen as discussed in \Cref{subsec:evaluation-of-lrt} and \Cref{sec:numerical-examples}. We also show in \Cref{sec:numerical-examples} that the proportion of terms that violate the penalized beta-binomial model assumptions is very small and we provide practical strategies to handle such terms.

\begin{figure}[!ht]
\centering
\includegraphics[width=\linewidth]{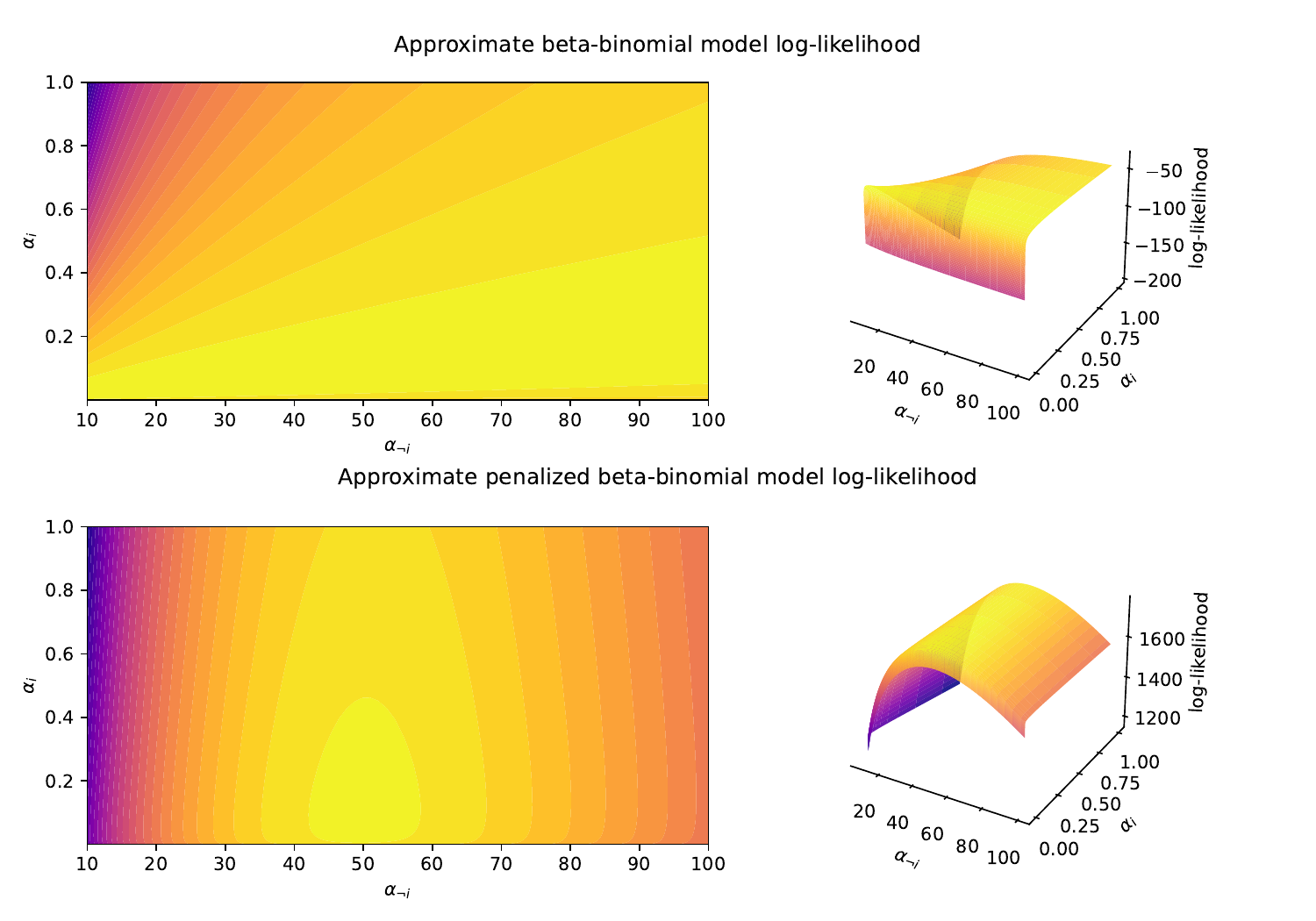}
\caption{Comparison of the approximate beta-binomial and penalized beta-binomial language models on a simulated document collection of size $d = 40$, with all documents having equal length $n_j = 50$ $(\forall j)$, and true parameter values $\alpha_i = 0.10$ and $\alpha_{\neg i} = 49.90$. The log-likelihood contour (upper left) and surface (upper right) plots for the beta-binomial approximation show a pronounced ridge structure. In contrast, the corresponding plots for the penalized beta-binomial approximation (bottom left and bottom right) show that the gamma penalty term breaks the ridge by regularizing the objective function. Gamma parameter values of $\mu=50$ and $\sigma^2=4$ were used to generate the plots. The choice $\sigma^2 = 4$ was made to make the breaking of the ridge visually apparent. However, alternative values of $\sigma^2$ produce the same qualitative effect, differing only in the steepness of the surface gradient.}
\label{fig:contour-v-surface}
\end{figure}

\subsection{Illustration of the effect of the gamma penalty term on parameter estimation} \label{sec:gbb-param-estimation}
To illustrate the model issue, we simulate a collection of $d=40$ documents, each containing $n_j=50$ words drawn from the beta-binomial distribution with $\alpha_i = 0.10$ and $\alpha_{\neg i}=49.90$. We then generate log-likelihood surface and contour plots as functions of $\alpha_i$ and $\alpha_{\neg i}$, both with and without the gamma penalty term on the precision parameter $\alpha_{0i}$, and compare the resulting estimates $\hat\alpha_i$ and $\hat\alpha_{\neg i}$ to the true underlying values. In evaluating the log-likelihood functions we set $\mu= 50$ (i.e., the mean document length) and $\sigma^2=4$. It is worth noting that the parameter $\sigma^2$ controls the surface gradient steepness of the gamma penalty term, and the value used here is chosen solely for visualization purposes to illustrate the breaking of the ridge from the unconstrained likelihood.

Figure~\ref{fig:contour-v-surface} presents this comparison. The top row shows the beta-binomial model approximation log-likelihood surface. A clear ridge appears along lines of roughly constant $\alpha_{0i}$, underscoring how this model is unidentifiable with respect to $\alpha_{0i}$ and how it struggles to distinguish between different combinations of $\alpha_i$ and $\alpha_{\neg i}$ that roughly sum to the same value. Without additional information, the true parameters $\alpha_i$ and $\alpha_{\neg i}$ cannot be accurately recovered. The bottom row corresponds to our approximation to the penalized beta-binomial model. The resulting log-likelihood surface displays a well-defined maximum, breaking the ridge yielding a unique penalized maximizer. The gamma penalty term effectively encourages values of $\alpha_{0i}$ near the gamma distribution mean (set to the average document length), resolving the ambiguity in $\alpha_{0i}$. It can be seen that the resulting penalized maximum likelihood estimates $\hat\alpha_i \approx 0.10$ and $\hat\alpha_{\neg i}\approx 49.77$ closely correspond to their respective true values.

\section{A penalized likelihood-ratio test statistic for word burstiness}\label{sec:likelihood-ratio-test}
In this section, we introduce a PLR test statistic for quantifying word burstiness and show that two variants of the TF–IDF term-weighting scheme arise as key components of this construction.

\subsection{Penalized likelihood-ratio test formulation}
\label{subsec:lrt-formulation}
The goal is to quantify the degree to which a target term, $t_i$, exhibits burstiness (or equivalently is over-dispersed) across the documents of a collection, $\mathcal{C}_i$, relative to what would be expected by chance. To this end, we define the following null and alternative hypotheses: 
\begin{description}
    \item[$\mathcal{H}_0$ ($t_i$ is non-bursty):] The  target term $t_i$'s document-level occurrences follow the binomial language model with $\theta_i = \hat\theta_i = n_i/n$.
    \item[$\mathcal{H}_1$ ($t_i$ is bursty):] The target term $t_i$'s document-level occurrences follow the approximate penalized beta-binomial model with $\alpha_i = \hat\alpha_i$ and $\alpha_{\neg i} = \hat\alpha_{\neg i}$.
\end{description}
Here, the binomial model serves as the null model. It assumes a single fixed probability of term occurrence across all documents. This model is too simplistic to capture burstiness. In contrast, the penalized beta-binomial model explicitly allows the term probability to vary between documents, therefore capturing burstiness in term counts, by which we mean the term counts are over-dispersed relative to what would be expected under the binomial model.

The penalized beta-binomial model includes a gamma penalty term on the precision parameter $\alpha_{0i}$ to discourage deviations from the mean document length, which is achieved by setting the gamma mean $\mu$ to the mean document length, as discussed in ~\Cref{subsec:gamma-penalty}. Consequently, the resulting ratio forms a PLR test statistic rather than a classical likelihood-ratio test statistic between two ordinary likelihoods. The PLR test statistic is defined as
\begin{equation}
        \lambda_i \defeq -2 \left( \ell_0(\hat{\theta}) - \ell_1(\hat{\boldsymbol{\alpha}}) \right),
\end{equation}
where $\ell_0(\hat{\theta})$ is the binomial log-likelihood evaluated at $\hat{\theta}_i = n_i/n$, and $\ell_1(\hat{\alpha})$ is the penalized beta-binomial log-likelihood from Eq.~\eqref{eq:gbb-log-likelihood} evaluated at $\boldsymbol{\hat{\alpha}}$ from Eqs.~\eqref{eq:mle-with-a0}~and~\eqref{eq:a0-mle-pen}.

The construction we describe does not constitute a test in the classical statistical sense, as we do not provide a concrete inferential procedure for computing a $p$-value to assess the statistical significance of word burstiness. Instead, our focus is on using the PLR test statistic as a scoring function for quantifying word burstiness, rather than on carrying out the associated hypothesis test. A formal testing procedure could, in principle, be obtained via a parametric bootstrap under the binomial null model.

\subsection{Evaluation of the test statistic and its relationship to TF–IDF}
\label{subsec:evaluation-of-lrt}
In this subsection, we derive the PLR test statistic for word burstiness and show that it can be approximated as a sum whose leading terms are common TF–IDF variants.

\begin{claim}\label{cl:cl3}
    The $\varhyphen{TF–IDF}$ variant term-weighting schemes $\varhyphen{BTF–IDF}$ and $\varhyphen{TF–ICF}$ naturally emerge from a PLR test statistic for the target term $t_i$ comparing a binomial null model with the approximate penalized beta-binomial alternative model of Eq.~\eqref{eq:gbb-approx-pmf}. The resulting test statistic $\lambda_i$ can be approximated as
        \begin{equation} \label{eq:lambda-tfidf}
        \begin{split}
            \lambda_i &\approx 2 \sum_{j=1}^d \Bigg[\varhyphen{TF–ICF}(i,j) - \varhyphen{BTF–IDF(i,j)} + \log\left(n_{ij}!/{n_{ij}}^{b_{ij}}\right) \Bigg] \\
        &\quad
            + 2n\log\frac{n}{b_i+n_{\neg i}}
            + 2(n_i - b_i) \log \frac{b_i + n_{\neg i}}{d\sigma^2}
            - (2b_i + 1)\log\mu \\
        &\quad
            + (\eta^2 - 2 r_i + 1)\log(\eta^2 - r_i) + r_i - \left(\eta^2 - 3/2\right)\log\eta^2 - \log\sqrt{2\pi},
        \end{split}
    \end{equation}    
    so long as $\eta^2 - r_i$ is positive, where $\eta = \mu/\sigma$ denotes the inverse coefficient of variation of the gamma distributed precision quantity $\Alpha_{0i}$.
\end{claim}

\begin{proof}[Derivation]
    We start from the general PLR test statistic
    \begin{equation}
        \lambda_i = -2 \left( \ell_0(\hat{\theta}) - \ell_1(\hat{\boldsymbol{\alpha}}) \right).
    \end{equation}
    Evaluating the binomial log-likelihood $\ell_0(\hat{\theta})$ at $\hat{\theta}_i = n_i/n$ yields
    \begin{equation*}
        \sum_{j=1}^d \log \binom{n_j}{n_{ij}} + n_{ij} \log \hat{\theta}_i + n_{\neg ij} \log \hat{\theta}_{\neg i}.
    \end{equation*}
    Evaluating the penalized beta-binomial log-likelihood, $\ell_1(\hat{\alpha})$, from Eq.~\eqref{eq:gbb-log-likelihood} at $\boldsymbol{\hat{\alpha}}$ gives the results in Eqs.~\eqref{eq:mle-with-a0}~and~\eqref{eq:a0-mle-pen}. Substituting these expressions, the PLR test statistic becomes
    \begin{equation}
        \begin{split}
        \lambda_i &= -2 \sum_{j=1}^d \Bigg[\log n_j! - \log n_{ij}! - \log n_{\neg{ij}}! + n_{ij}\log \hat{\theta}_i + n_{\neg ij} \log \hat{\theta}_{\neg i} \\
        &\quad - \log n_j! + \log n_{\neg{ij}}! - b_{ij} \log \hat{\alpha}_i + b_{ij} \log n_{ij} - n_{\neg{ij}} \log \hat{\alpha}_{\neg i} + n_j \log \hat{\alpha}_{0i} \Bigg] \\
        &\quad + \log f(\hat{\alpha}_{0i}; \mu, \sigma^2).
        \end{split}
    \end{equation}
    Using $\hat{\theta}_i = n_i/n$, this simplifies to
    \begin{equation}
        \begin{split}
        \lambda_i &= -2 \sum_{j=1}^d \Bigg[ n_{ij}\log\frac{n_i}{n} + (n_j-n_{ij})\log\left(\frac{n-n_i}{n}\right)-\log n_{ij}! - b_{ij}\log{\hat{\alpha}_i} \\
        &\quad + b_{ij}\log{n_{ij}} - (n_j-n_{ij})\log \hat{\alpha}_{\neg i} + n_j \log \hat{\alpha}_{0i} \Bigg] + \log f(\hat{\alpha}_{0i}; \mu, \sigma^2).
        \end{split}
    \end{equation}
    Grouping terms and applying logarithmic identities yields
    \begin{align}
        \begin{split}
            \lambda_i &= -2 \sum_{j=1}^d \Bigg[ n_{ij} \left(\log\frac{n_i}{n} - \log\frac{n-n_i}{n} + \log \hat{\alpha}_{\neg i} \right) \\
            &\quad + n_j \left(\log\frac{n-n_i}{n} + \log \hat{\alpha}_{0i} - \log \hat{\alpha}_{\neg i} \right) \\
            &\quad + b_{ij} \left(\log n_{ij} - \log \hat{\alpha}_i\right) - \log n_{ij}! \Bigg] + \log f(\hat{\alpha}_{0i}; \mu, \sigma^2)
        \end{split} \\
        \begin{split}
            &= -2 \sum_{j=1}^d \Bigg[ n_{ij} \log\frac{n_i \hat{\alpha}_{\neg i}}{n-n_i} + n_j \log\frac{\hat{\alpha}_{0i}(n-n_i)}{n\alpha_{\neg i}} + b_{ij}\log\frac{n_{ij}}{\hat{\alpha}_i} - \log n_{ij}! \Bigg] \\
            &\quad + \log f(\hat{\alpha}_{0i}; \mu, \sigma^2).
        \end{split}
    \end{align}
    Writing $\hat{\alpha}_i$ and $\hat{\alpha}_{\neg{i}}$ in terms of $\hat{\alpha}_{0i}$ from Eq.~\eqref{eq:bcb-penalized-mles}, we obtain
    \begin{equation}
    \label{eq:lambda-exact}
        \begin{split}
        \lambda_i &= -2 \sum_{j=1}^d \Bigg[
            n_{ij} \log\frac{n_i \hat{\alpha}_{0i} }{b_i + n_{\neg i}}
            + n_j \log\frac{b_i+n_{\neg i}}{n} + b_{ij} \log\frac{n_{ij}(b_i+n_{\neg i})}{b_i \hat{\alpha}_{0i}} - \log n_{ij}! \Bigg] \\
        &\quad 
            + \log f(\hat{\alpha}_{0i}; \mu, \sigma^2)
        \end{split}
    \end{equation}
    Finally, substituting $\hat{\alpha}_{0i}$ from Eq.~\eqref{eq:a0-mle-pen} gives
    \begin{equation}
    \label{eq:pre-mu}
        \begin{split}
            \lambda_i &= -2 \sum_{j=1}^d \Bigg[ n_{ij} \log\frac{n_i \sigma^2(b_i-n_i-1+\mu^2/\sigma^2)}{\mu(b_i + n_{\neg i})} + n_j\log\frac{b_i+n_{\neg i}}{n} \\
        &\quad
            + b_{ij}\log\frac{\mu n_{ij}(b_i+n_{\neg i})}{b_i \sigma^2 (b_i-n_i-1+\mu^2/\sigma^2)} - \log n_{ij}! \Bigg] + \log f(\hat{\alpha}_{0i}; \mu, \sigma^2).
        \end{split}
    \end{equation}
    We isolate the $n_{ij}\log(\cdot)$ term
    \begin{eqnarray*}
         n_{ij} \log\frac{n_i \sigma^2(b_i-n_i-1+\mu^2/\sigma^2)}{\mu(b_i + n_{\neg i})} &=& n_{ij} \log \frac{n_i}{\mu} + n_{ij} \log \frac{b_i-n_i-1+\mu^2/\sigma^2}{b_i+n_{\neg i}} + n_{ij}\log\sigma^2.
    \end{eqnarray*}
    Using $\mu = n/d$, we can write
    \begin{equation*}
        n_{ij} \log \frac{n_i}{n} + n_{ij}\log{d} + n_{ij} \log \frac{b_i-n_i-1+\mu^2/\sigma^2}{b_i+n_{\neg i}} + n_{ij}\log\sigma^2 .
    \end{equation*}
    We realize that first term is $-\varhyphen{TF–ICF}(i,j)$ and combine the logarithmic terms. The resulting expression is
    \begin{equation*}
        -\varhyphen{TF–ICF}(i,j) + n_{ij} \log \frac{d\sigma^2(b_i-n_i-1+\mu^2/\sigma^2)}{b_i+n_{\neg i}}.
    \end{equation*}
    Substituting this back into the main summation gives
    \begin{equation}
        \begin{split}
            \lambda_i &= -2 \sum_{j=1}^d \Bigg[ -\varhyphen{TF–ICF}(i,j) + n_{ij} \log \frac{d\sigma^2(b_i-n_i-1+\mu^2/\sigma^2)}{b_i+n_{\neg i}} + n_j\log\frac{b_i+n_{\neg i}}{n} \\
        &\quad
            + b_{ij}\log\frac{\mu n_{ij}(b_i+n_{\neg i})}{b_i \sigma^2 (b_i-n_i-1+\mu^2/\sigma^2)} - \log n_{ij}! \Bigg] + \log f(\hat{\alpha}_{0i}; \mu, \sigma^2).
        \end{split}
    \end{equation}
    Now, we isolate the $b_{ij}\log(\cdot)$ term and use $\mu=n/d$ to obtain
    \begin{equation*}
        b_{ij}\log\frac{\mu n_{ij}(b_i+n_{\neg i})}{b_i \sigma^2 (b_i-n_i-1+\mu^2/\sigma^2)} = b_{ij}\log\frac{n_{ij}}{b_i} + b_{ij}\log\frac{n(b_i+n_{\neg i})}{d \sigma^2 (b_i-n_i-1+\mu^2/\sigma^2)}.
    \end{equation*}
    Adding $\log{(d/d)}$ and using logarithmic properties gives
    \begin{equation*}
         b_{ij}\log\frac{d}{b_i} + b_{ij}\log\frac{nn_{ij}(b_i+n_{\neg i})}{d^2 \sigma^2 (b_i-n_i-1+\mu^2/\sigma^2)}.
    \end{equation*}
    We realize that the first term is $\varhyphen{BTF–IDF}(i,j)$ from Eq.~\eqref{eq:btfidf}. We also introduce the inverse coefficient of variation $\eta = \mu / \sigma$ and use $r_i = n_i - b_i + 1$ as defined in Table~\ref{tab:bag-of-words-model}. The equation can be rewritten as
    \begin{equation*}
        \varhyphen{BTF–IDF}(i,j) + b_{ij}\log\frac{nn_{ij}(b+n_{\neg i})}{d^2\sigma^2(\eta^2-r_i)}.
    \end{equation*}
    Substituting this back into the main summation, we find
    \begin{equation}
        \begin{split}
            \lambda_i &= -2 \sum_{j=1}^d \Bigg[ -\varhyphen{TF–ICF}(i,j) + \varhyphen{BTF–IDF(i,j)} + n_{ij}\log\frac{d\sigma^2(b_i-n_i-1+\mu^2/\sigma^2)}{b_i+n_{\neg i}} \\
        &\quad
            + n_j\log\frac{b_i+n_{\neg i}}{n}
            + b_{ij}\log\frac{nn_{ij}(b_i+n_{\neg i})}{d^2 \sigma^2 (b_i-n_i-1 + \mu^2/\sigma^2)} - \log n_{ij}! \Bigg]  \\
        &\quad + \log f(\hat{\alpha}_{0i}; \mu, \sigma^2).
        \end{split}
    \end{equation}
    We compute the terms in the summation and we evaluate the logarithm of the gamma penalty term (see Appendix~\ref{appn:log-gamma}). We get
    \begin{equation}
        \begin{split}
            \lambda_i &\approx 2 \sum_{j=1}^d \Bigg[\varhyphen{TF–ICF}(i,j) - \varhyphen{BTF–IDF}(i,j) + \log\left(n_{ij}!/{n_{ij}}^{b_{ij}}\right) \Bigg] \\
        &\quad
            + 2n\log\frac{n}{b_i+n_{\neg i}}
            + 2n_i \log \frac{b_i + n_{\neg i}}{d\sigma^2(\eta^2-r_i)}
            - 2b_i\log\frac{n(b_i+n_{\neg i})}{d^2 \sigma^2 (\eta^2-r_i)} \\
        &\quad
            + (\eta^2 - 1)\log(\eta^2 - r_i) - \log\mu + r_i - \left(\eta^2 - 3/2\right)\log\eta^2 - \log\sqrt{2\pi}.
        \end{split}
    \end{equation}
    Collecting like terms and substituting $\mu=n/d$ as needed yields
        \begin{equation}
        \begin{split}
            \lambda_i &\approx 2 \sum_{j=1}^d \Bigg[\varhyphen{TF–ICF}(i,j) - \varhyphen{BTF–IDF(i,j)} + \log\left(n_{ij}!/{n_{ij}}^{b_{ij}}\right) \Bigg] \\
        &\quad
            + 2n\log\frac{n}{b_i+n_{\neg i}}
            + 2(n_i - b_i) \log \frac{b_i + n_{\neg i}}{d\sigma^2}
            - (2b_i + 1)\log\mu \\
        &\quad
            + (\eta^2 - 2 r_i + 1)\log(\eta^2 - r_i) + r_i - \left(\eta^2 - 3/2\right)\log\eta^2 - \log\sqrt{2\pi}.
        \end{split}
    \end{equation}
    This completes the derivation.
    \end{proof}

We conclude this section with a brief discussion on $\lambda_i$. The terms inside the sum connect local statistics (i.e., document-level) with global statistics (i.e., collection-level). TF–ICF assigns high scores for rare terms in the collection, while BTF–IDF assigns high scores for terms that appear in a small fraction of documents. The third term, given the term appears in the document, evaluates to $\log\left(\left(n_{ij}-1\right)!\right)$ which is a TF-like term that adds weight for terms that occur frequently. For small to moderate term frequencies, this term behaves similarly to~TF. In the \textit{Reuters-21578} dataset, for instance, roughly 98\% of nonzero term counts lie in the range $n_{ij} \leq 7$, where $\log\left(\left(n_{ij}-1\right)!\right)$ closely tracks TF. As $n_{ij}$ increases beyond this range, the term exhibits asymptotic growth on the approximate order of $n_{ij}\log n_{ij} - n_{ij}$.

A practical consideration arises with the term $\eta^2 - r_i$ as it must remain positive (i.e., $\eta^2 = \mu^2/\sigma^2 > n_i - b_i + 1$). If $\sigma > \mu$, then $\eta^2 - r_i$ becomes negative regardless of the values of $n_i$ and $b_i$, since $n_i \geq b_i$. To avoid such issues in practice, one can set $\sigma^2 = 1$, which guarantees positivity as long as $n_i < \mu^2$. In our framework, this condition is generally satisfied by tuning $\mu$ as the mean document length. It is therefore imperative that we be mindful of this condition when setting the value of $\sigma$ in subsequent experiments. Note that all results following Eq.~\eqref{eq:pre-mu} use $\mu = n/d$, leaving $\sigma^2$ as the sole tunable parameter. In the extreme cases where $n_i \geq \mu^2$ (i.e, the total count of a term in the collection exceeds the mean document length squared) one can take $\max(1, \eta^2-r_i)$ which evaluates the term $(\eta^2-2r_i+1)\log(\eta^2-r_i)$ to zero. Reparameterization is also a possible alternative to handle such extreme cases.

\section{Numerical examples}
\label{sec:numerical-examples}
In this section, we explore applying the derived PLR test statistic as a term-weighting scheme and compare it with TF–IDF using both synthetic and real-world data.

\subsection{A simulated example}
\label{subsec:numerical-synthetic}
Here we simulate a synthetic collection of documents where each term is represented as a beta-binomial model to investigate how~$\lambda_i$ relates to TF–IDF weights.

Using a vocabulary of~$m=250$ terms we draw~$\boldsymbol{\alpha}$ from a skewed gamma distribution concentrated near zero, with most values very small but a long tail allowing a few relatively larger values as typical of real documents~\citep[see][]{Elkan2006}. Moreover, we introduce some variability in document lengths~$n_j$ by drawing~$n_j$ from a binomial distribution centered around~$n_j=75$ rather than using a fixed value. Given this modeling setup, for each document we draw term proportions~$\boldsymbol{\theta_j}$ from $\text{Dir}(\boldsymbol{\alpha})$, then sample term counts~$\boldsymbol{n_j}$ from $\text{Mult}(n_j, \boldsymbol{\theta_j})$. This generation process leverages the Dirichlet-multinomial model to produce synthetic documents reminiscent of those making up real-world corpora, as discussed in~\Cref{subsec:dm-model}.

The resulting collection is stored as a $d \times m$ matrix of term frequencies, with each entry~$(i,j)$ equal to the count $n_{ij}$ of term $t_i$ in document $d_j$. To match the structure assumed by the beta-binomial generative model, we represent each term count in a masked form. Specifically, for term~$t_i$ in document~$d_j$, we record the pair~$(n_{ij}, n_{\neg{ij}})$. This form facilitates the computation of the PLR test statistic~$\lambda_i$. To compute~$\lambda_i$ from Eq.~\eqref{eq:lambda-tfidf} we set~$\mu = \bar{n}_j$, where~$\bar{n}_j = 74.59$ is the sample mean of document lengths as discussed in ~\Cref{subsec:gamma-penalty} and $\sigma^2=1$. In practice, both $\mu$ and $\sigma^2$ can be tuned. Finally, we plot~$\lambda_i$ against the total TF–IDF for each term~$t_i$.

\begin{figure}[!ht]
\includegraphics[width=\linewidth]{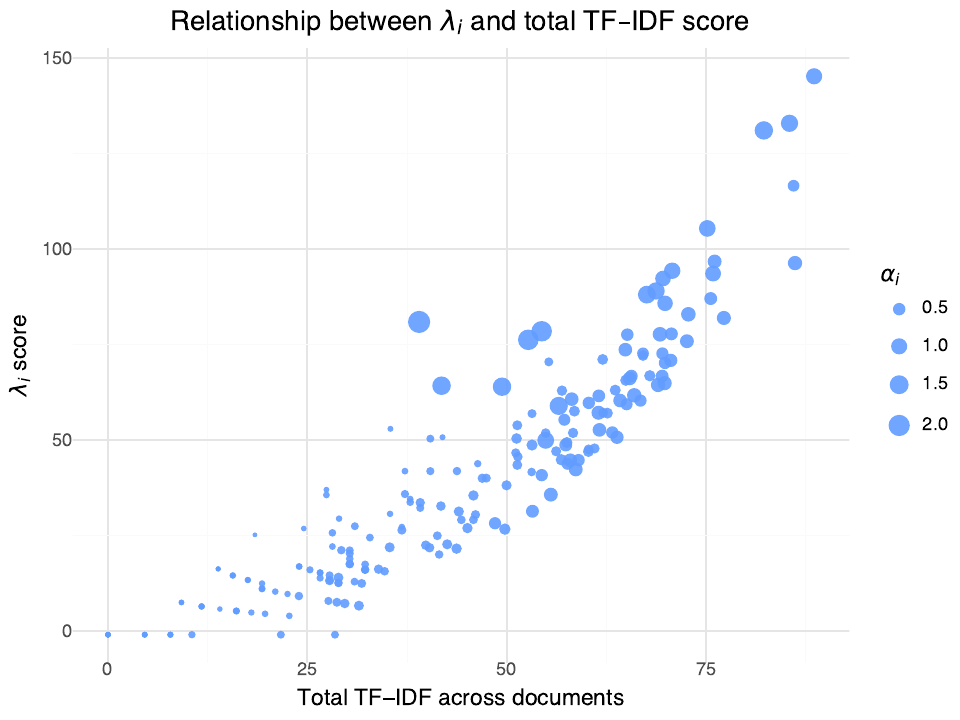}
\caption{Scatterplot of the relationship between the likelihood–ratio test score $\lambda_i$ and the total TF–IDF weight of term~$t_i$ across all collection documents. Point diameter is proportional to Dirichlet concentration parameter~$\alpha_i$ magnitude.}
\label{fig:lambda-vs-tfidf}
\end{figure}

Figure~\ref{fig:lambda-vs-tfidf} shows the relationship between the PLR test statistic $\lambda_i$ and the total TF–IDF of each term $t_i$ defined as $\sum_{j=1}^d{\varhyphen{TF–IDF}(i,j)}$. Each point corresponds to a term in the vocabulary, with point diameter proportional to the term's corresponding $\alpha_i$. The scatter plot shows a positive correlation where terms with higher TF–IDF weights also tend to have larger $\lambda_i$ values. In this simulation, the correlation coefficient is $r=0.9269$, rounded to four decimal places, which indicates that the formulation of $\lambda_i$ aligns closely with the empirical weighting captured by TF–IDF. It can be seen that some points in the plot appear as outliers. These arise because a few $\alpha_i$ values are relatively large, which violate the assumptions of the penalized beta-binomial model underlying the derivation of $\lambda_i$ (i.e., $\alpha_i \ll 1$). However, these outliers have a limited impact on the reliability of $\lambda_i$ as we show in the following real-data examples that the proportion of terms with a corresponding $\alpha_i$ violating this condition is in practice very small.

\subsection{The 20 Newsgroups dataset example}
\label{subsec:numerical-real}
In this subsection, we evaluate the term-weighting scheme motivated by our PLR test statistic on the \textit{20~Newsgroups} dataset, a widely used benchmark for text classification curated by Lang~\cite{Lang1995}. The collection contains approximately 20,000 Usenet posts distributed across 20 topics. We analyze the preprocessed version included in the Python \texttt{scikit-learn} library, consisting of 18,846 posts. Additional preprocessing was applied to convert all text to lowercase, remove punctuation, non-alphabetic characters, and stopwords. Additionally, multiple whitespace was collapsed into single spaces. After preprocessing, the dataset contained 88,701 unique terms. We split the documents into training (60\%) and testing (40\%) subsets for our experiments.

Before training the classifier, we empirically verify that the vocabulary broadly satisfies the theoretical conditions required for the beta-binomial approximation derived in Claim~\ref{claim:cl1}, namely that the fitted parameters satisfy $a_i \ll 1$ and $a_{\neg i} \gg 1$ for the vocabulary terms.

To do so, we fit a beta-binomial distribution independently for each term using the target-complement document representation of the dataset, where each document $d_j$ is represented as the ordered pair $(n_{ij},\, n_{\neg{ij}})$. Parameter estimation was carried out via maximum likelihood using the SciPy Python library function \texttt{scipy.optimize.minimize} with a uniform initialization of $(\alpha_i, \alpha_{\neg i}) = (1, 1)$. Stopwords we deliberately retained in order to contrast their behavior with other terms.

The fitted parameters exhibit a strong concentration in the theoretically predicted regime: 99.04\% of terms satisfy $\alpha_i < 0.1$, while 98.89\% satisfy $\alpha_{\neg i} > 10$. These findings largely validate the assumptions of Claim~\ref{claim:cl1}, namely that the target-term concentration parameter is sufficiently small, whereas the complement concentration parameter is sufficiently large.

\begin{figure}[!ht]
    \centering
    \includegraphics[width=\linewidth]{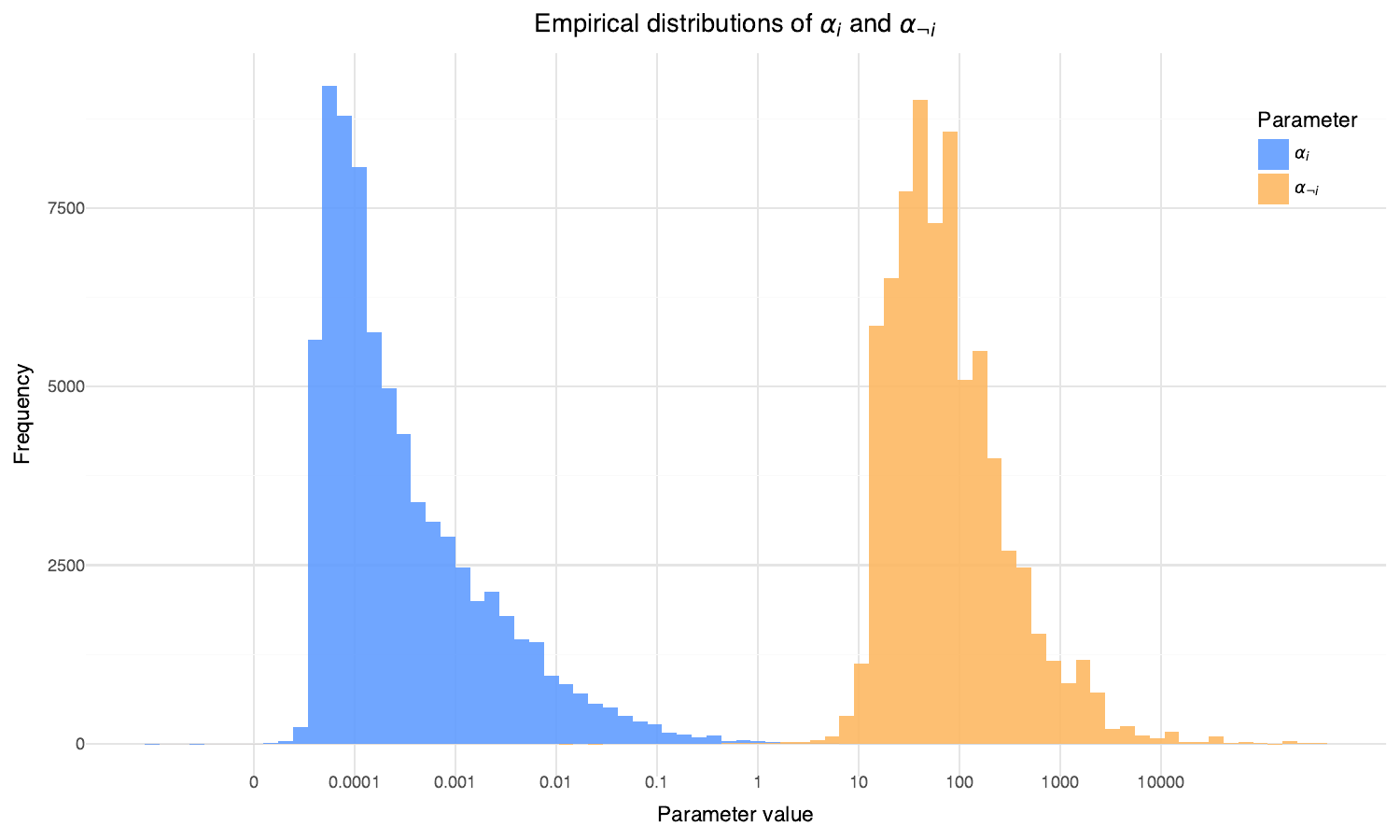}
    \caption{Empirical distributions of the fitted beta-binomial parameters $\alpha_i$ and $\alpha_{\neg i}$ for the \textit{20~Newsgroups} dataset vocabulary, shown on a natural logarithmic scale. The pronounced concentrations of mass at $\alpha_i \ll 1$ and $\alpha_{\neg i} \gg 1$ support the assumptions of Claim~\ref{claim:cl1}.}
    \label{fig:empirical-alphas-20ng}
\end{figure}

Figure~\ref{fig:empirical-alphas-20ng} shows the empirical distributions of the fitted parameters $a_i$ and $a_{\neg i}$ on a natural logarithmic scale. As illustrative examples, the bursty and semantically specific term ``ambulance'' was assigned $(\alpha_i, \alpha_{\neg i}) = (0.0021,\, 128.30)$, while ``baby'' yielded $(0.0041,\, 83.48)$. In contrast, high-frequency function words exhibit substantially larger target parameters; for instance, ``the'' was assigned $(5.22,\, 93.38)$ and ``for'' $(3.33,\, 306.70)$.

Motivated by the PLR test statistic approximation derived in Claim~\ref{cl:cl3}, we define a new term-weighting scheme, $\mathcal{S}(\lambda_{ij})$, where $\mathcal{S}(\cdot)$ denotes the sigmoid function, and 
\begin{equation}
\label{eq:lambda-ij}
    \begin{split}
        \lambda_{ij} &\defeq
            \varhyphen{TF–ICF}(i,j)
            - \varhyphen{BTF–IDF}(i,j)
            + \log \left(n_{ij}!/{n_{ij}}^{b_{ij}}\right)  \\
        &\quad
            + (n_{ij} - b_{ij}) \log\frac{b_i + n_{\neg i}}{d \sigma^2}
            + n_j \log \frac{n}{b_i + n_{\neg i}}    
            - \left(b_{ij} + \frac{1}{2d}\right) \log\mu \\
        &\quad
            + \frac{1}{2d}
            \left(
              (\eta^2 - 2r_i + 1) \log\left(\max(1 , \eta^2 - r_i)\right) 
              - (\eta^2 - 3/2) \log \eta^2
              + r_i
            \right).
    \end{split}
\end{equation}
In this formulation, Eq.~\eqref{eq:lambda-tfidf} can be rewritten as $\lambda_i \approx 2 \sum_{j=1}^d \lambda_{ij} - \log\sqrt{2\pi}$, so that~$\lambda_{ij}$ represents the contribution of term~$t_i$ in document~$d_j$, neglecting the constant contribution from the $-\log\sqrt{2\pi}$ term. To handle any troublesome terms satisfying $\eta^2-r_i \leq 0$, we implant the maximum of $\eta^2-r_i$ and $1$ inside the relevant logarithm. Such cases may alternatively be obviated by reparameterizing $\sigma$ to be sufficiently small so as to ensure that the $\eta^2 - r_i > 0$ condition is satisfied.

Since $\lambda_{ij}$ may take on negative values, we apply the sigmoid transformation $\mathcal{S}(\lambda_{ij}) = 1/(1+e^{-\lambda_{ij}})$ to map the weights to the open unit interval. In this experiment, the parameter~$\mu$ was tuned to the preprocessed corpus mean document length of the training set of 165.29, while~$\sigma^2$ was fixed at $1$. This choice can be used as a starting point in practice. The choice of $\sigma^2=1$ is to mitigate against troublesome words satisfying $\eta^2-r_i \leq 0$ as discussed in \Cref{sec:gbb-mles,subsec:evaluation-of-lrt}. In practical settings, both hyperparameters can be tuned to optimize performance. However, our goal is to demonstrate the performance of $\mathcal{S}(\lambda_{ij})$ when $\mu$ is set to the mean document length and to avoid values of $\sigma^2$ which result in $\eta^2-r_i \leq 0$. This setup is convenient insofar as it makes use of all of the terms occurring the in formula of Eq.~\eqref{eq:lambda-tfidf}.

We refer the interested reader to Fig.~\ref{fig:20ng-sensitivity-analysis} in ~\Cref{appn:sensitivity-analyses} which reports the overall accuracy of the classifier across selected values of $\mu$ and $\sigma^2$. This sensitivity analysis shows that the accuracy remains stable across non-extreme parameter values.

\begin{table}[!ht]
\centering
\caption{Multinomial Naive Bayes classification results on the \textit{20~Newsgroups} dataset using TF–IDF and $\mathcal{S}(\lambda_{ij})$ features.}
\begin{tabular}{lcccccc}
 & \multicolumn{3}{c}{\textbf{TF–IDF}} & \multicolumn{3}{c}{$\boldsymbol{\mathcal{S}(\lambda_{ij})}$} \\
\cmidrule(lr){2-4} \cmidrule(lr){5-7}
\textbf{Document Class} & Precision & Recall & F1 & Precision & Recall & F1 \\
\midrule
alt.atheism & 0.82 & 0.82 & 0.82 & 0.79 & 0.76 & 0.78 \\
comp.graphics & 0.61 & 0.76 & 0.68 & 0.74 & 0.69 & 0.71 \\
comp.os.ms-windows.misc & 0.86 & 0.10 & 0.17 & 0.84 & 0.45 & 0.58 \\
comp.sys.ibm.pc.hardware & 0.52 & 0.78 & 0.62 & 0.62 & 0.76 & 0.69 \\
comp.sys.mac.hardware & 0.73 & 0.85 & 0.78 & 0.86 & 0.75 & 0.80 \\
comp.windows.x & 0.80 & 0.75 & 0.77 & 0.70 & 0.85 & 0.76 \\
misc.forsale & 0.78 & 0.74 & 0.76 & 0.90 & 0.71 & 0.79 \\
rec.autos & 0.87 & 0.91 & 0.89 & 0.88 & 0.90 & 0.89 \\
rec.motorcycles & 0.91 & 0.96 & 0.94 & 0.96 & 0.93 & 0.95 \\
rec.sport.baseball & 0.94 & 0.94 & 0.94 & 0.96 & 0.89 & 0.93 \\
rec.sport.hockey & 0.96 & 0.97 & 0.96 & 0.90 & 0.98 & 0.94 \\
sci.crypt & 0.87 & 0.94 & 0.90 & 0.68 & 0.96 & 0.80 \\
sci.electronics & 0.77 & 0.75 & 0.76 & 0.79 & 0.67 & 0.73 \\
sci.med & 0.91 & 0.82 & 0.86 & 0.89 & 0.82 & 0.85 \\
sci.space & 0.89 & 0.90 & 0.89 & 0.80 & 0.93 & 0.86 \\
soc.religion.christian & 0.88 & 0.92 & 0.90 & 0.76 & 0.95 & 0.85 \\
talk.politics.guns & 0.79 & 0.90 & 0.84 & 0.69 & 0.93 & 0.79 \\
talk.politics.mideast & 0.97 & 0.90 & 0.94 & 0.90 & 0.95 & 0.93 \\
talk.politics.misc & 0.74 & 0.66 & 0.70 & 0.73 & 0.61 & 0.66 \\
talk.religion.misc & 0.70 & 0.66 & 0.68 & 0.88 & 0.35 & 0.50 \\
\midrule
\textbf{Average (macro)} & 0.82 & 0.80 & 0.79 & 0.81 & 0.79 & 0.79 \\
\textbf{Average (weighted)} & 0.82 & 0.81 & 0.79 & 0.81 & 0.80 & 0.80 \\
\textbf{Accuracy} & & & 0.81 & & & 0.80 \\
\end{tabular}
\label{tab:20ng-lambda-results}
\end{table}

Table~\ref{tab:20ng-lambda-results} reports the classification results using a multinomial Naive Bayes classifier trained on each of TF–IDF and $\mathcal{S}(\lambda_{ij})$ features. For each class, precision, recall, and F1 scores are shown. TF–IDF achieves macro-averaged precision, recall, and F1 of 0.82, 0.80, and 0.79, with overall accuracy 0.81. The proposed $\mathcal{S}(\lambda_{ij})$ term-weighting features achieve macro-averaged precision, recall, and F1 of 0.81, 0.79, and 0.79, with overall accuracy 0.80. The weighted averages of precision, recall, and F1 are very close to the macro-averaged scores, because the \textit{20~Newsgroups} dataset is relatively balanced across its classes.

Looking more closely, $\mathcal{S}(\lambda_{ij})$ outperforms or matches TF–IDF in 9 out of 20 classes in precision with the largest improvement in precision observed in the \texttt{talk.religion.misc} class (+0.18). In recall, $\mathcal{S}(\lambda_{ij})$ is greater than or equal to TF–IDF in 9 classes with the largest gain in the \texttt{comp.os.ms-windows.misc} class (+0.35). In terms of F1 score, $\mathcal{S}(\lambda_{ij})$ is greater than or equal to TF–IDF in 7 classes, with the largest improvement again in the \texttt{comp.os.ms-windows.misc} class (+0.41). These results indicate that $\mathcal{S}(\lambda_{ij})$ can be comparable to TF–IDF for term-weighting on document classification tasks, albeit with a higher computational overhead.

\subsection{The R8 dataset example}\label{subsec:numerical-real-ii}
In this subsection, we further evaluate our proposed method on the \textit{R8} dataset~\cite{CardosoCachopo2007}, a collection of single-class labeled documents extracted from the \textit{Reuters-21578} text categorization collection~\cite{Lewis1987}. This dataset focuses on the eight most frequent \textit{Reuters-21578} classes: \texttt{acq}, \texttt{crude}, \texttt{earn}, \texttt{grain}, \texttt{interest}, \texttt{money-fx}, \texttt{ship}, and \texttt{trade}. The collection is highly imbalanced, with the most frequent class (i.e., \texttt{earn}) accounting for approximately 49\% of all collection documents. We applied the same preprocessing steps as described in \Cref{subsec:numerical-real}. Specifically, text was converted to lowercase, while non-alphabetic characters and stopwords were removed. After preprocessing, the collection contained 19,688 unique terms. We use the standard predefined train-test split, containing 5,501 training documents and 2,190 testing documents. 

Before training the classifier, we empirically verify that the vocabulary satisfies the theoretical conditions required for the beta-binomial approximation derived in Claim~\ref{claim:cl1}, namely that the fitted parameters satisfy $\alpha_i \ll 1$ and $\alpha_{\neg i} \gg 1$ for the majority of terms using the same parameter estimation methodology from \Cref{subsec:numerical-real}. 

\begin{figure}[!ht]
    \centering
    \includegraphics[width=\linewidth]{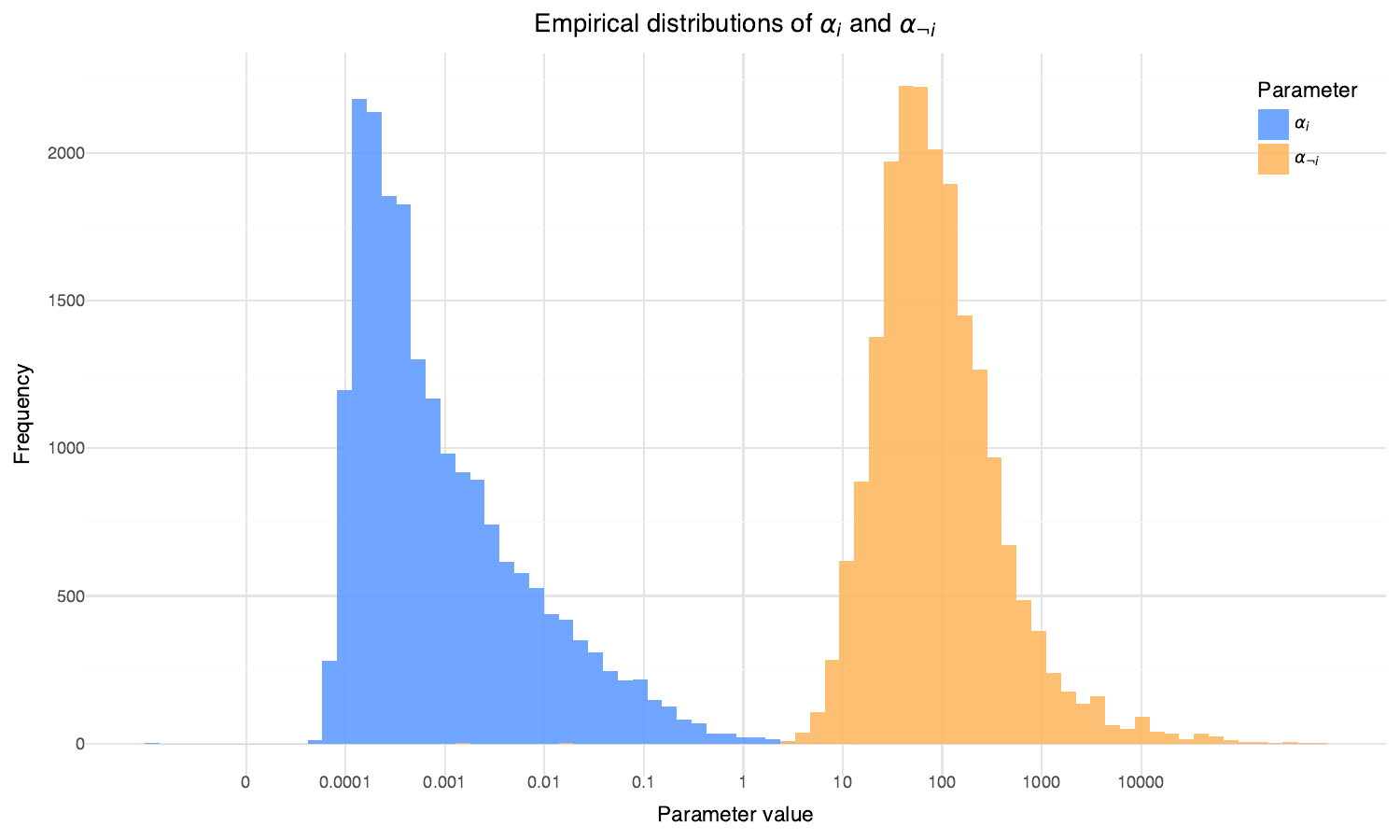}
    \caption{Empirical distributions of the fitted beta-binomial parameters $\alpha_i$ and $\alpha_{\neg i}$ for the \textit{R8} dataset vocabulary, shown on a natural logarithmic scale. The pronounced concentrations of mass at $\alpha_i \ll 1$ and $\alpha_{\neg i} \gg 1$ support the assumptions of Claim~\ref{claim:cl1}.}
    \label{fig:empirical-alphas-r8}
\end{figure}

Figure~\ref{fig:empirical-alphas-r8} shows the distributions of the fitted parameters $\alpha_i$ and $\alpha_{\neg i}$ on a natural logarithmic scale. These parameters are consistent with the assumptions of Claim~\ref{claim:cl1}: 96.95\% of terms satisfy $\alpha_i < 0.1$, while 97.19\% satisfy $\alpha_{\neg i} > 10$. As illustrative examples, the bursty and semantically specific term ``student'' was assigned $(\alpha_i, \alpha_{\neg i}) = (0.00025,\, 63.37)$, while ``program'' yielded $(0.0396,\, 170.92)$. In contrast, high-frequency function words exhibit substantially larger target parameters; for instance, ``told'' was assigned $(1.08,\, 982.11)$ and ``said'' $(8.76,\, 437.97)$.

\begin{table}[!ht]
\centering
\caption{Multinomial Naive Bayes classification results on the \textit{R8} dataset using TF–IDF and $\mathcal{S}(\lambda_{ij})$ features.}
\begin{tabular}{lcccccc}
 & \multicolumn{3}{c}{\textbf{TF–IDF}} & \multicolumn{3}{c}{$\boldsymbol{\mathcal{S}(\lambda_{ij})}$} \\
\cmidrule(lr){2-4} \cmidrule(lr){5-7}
\textbf{Document Class} & Precision & Recall & F1 & Precision & Recall & F1 \\
\midrule
acq & 0.98 & 0.91 & 0.94 & 0.95 & 0.99 & 0.97 \\
crude & 0.90 & 0.93 & 0.92 & 0.91& 0.87& 0.89\\
earn & 0.95 & 0.97 & 0.96 & 0.99& 0.98& 0.98 \\
grain & 0.91 & 1.00 & 0.95 & 1.00 & 0.30 & 0.46 \\
interest & 0.87 & 0.74 & 0.80 & 0.98 & 0.52 & 0.68 \\
money-fx & 0.77 & 0.92 & 0.84 & 0.72& 0.90 & 0.80\\
ship & 0.76 & 0.78 & 0.77 & 1.00 & 0.33& 0.50\\
trade & 0.80 & 0.96 & 0.87 & 0.70 & 0.93 & 0.80 \\
\midrule
\textbf{Average (macro)} & 0.87 & 0.90 & 0.88 & 0.90& 0.73 & 0.76\\
\textbf{Average (weighted)} & 0.94 & 0.94 & 0.94 & 0.95 & 0.94 & 0.94 \\
\textbf{Accuracy} & & & 0.94 & & & 0.94 \\
\end{tabular}
\label{tab:r8-lambda-results}
\end{table}

Table~\ref{tab:r8-lambda-results} compares the performance of TF–IDF and $\mathcal{S}(\lambda_{ij})$ using a multinomial Naive Bayes classifier. In this experiment, we tuned $\mu$ to the mean document length $\mu=65.75$ and $\sigma^2=1$, consistent with our methodology discussed in \Cref{subsec:numerical-real}. We refer the interested reader to Fig.~\ref{fig:r8-sensitivity-analysis} in \Cref{appn:sensitivity-analyses} for an accompanying sensitivity analysis. Both methods achieve a high overall accuracy of 0.94. However, we notice that classes like \texttt{grain}, \texttt{ship}, and \texttt{interest} exhibits lower F1 scores. This can be attributed to these classes being of the smallest in the dataset, accounting for only 5.8\% of the documents combined, whereas more common classes like \texttt{acq}, \texttt{crude}, and \texttt{earn} achieve higher F1 scores. Consequently, the macro-average scores are significantly lower than the weighted-average scores, since the macro average treats minority classes as equally important. We also observe that TF–IDF outperforms $\mathcal{S}(\lambda_{ij})$ on these minority classes with small number of documents.

\section{Discussion} 
\label{sec:discussion}
This paper has taken initial steps towards toward establishing a statistical justification for the widely used TF–IDF term-weighting scheme. Our main contribution is to demonstrate that the TF–IDF variants BTF–IDF and TF–ICF arise naturally from a likelihood-ratio test statistic for quantifying word burstiness, formulated within a beta-binomial modeling framework for document representation. In an example of synthetic data generated by the Dirichlet-multinomial model, we observed a strong correlation between the PLR test statistic and corresponding TF–IDF term weights. Furthermore, we proposed a term-weighting scheme from the PLR test statistic and showed that it performed comparably to TF–IDF on document classification tasks. However, this term-weighting scheme's main contribution is providing a theoretical foundation rather than improving performance, especially considering the extra computation required per term.

Future research should explore generalizing the results to account for dependencies between words. For instance, the occurrences of words like ``Socrates'' and ``Plato'' are not independent since the likelihood of the word ``Plato'' occurring in a given document increases given the presence of the word ``Socrates''. Cummins~\cite{Cummins2017} proposed a framework that generalizes statistical language models used in document generation based on the choice of a replacement matrix $\mathbf{M}$. For example, using $\mathbf{M}=\mathbf{0}$ corresponds to the multinomial model, while $\mathbf{M}=\mathbf{1}$ corresponds to the Dirichlet-multinomial model. Cummins exclusively explored diagonal choices of the replacement matrix, but also notes that allowing $\mathbf{M}$ to be an arbitrary matrix could provide a way to capture dependencies between different word types within documents. We advocate further exploration of employing such generalized replacement matrices to model the word co-burstiness phenomenon described above.

Another direction of future work is to investigate different penalty terms in the penalized likelihood-ratio test and study how the resulting statistic corresponds to existing term-weighting schemes~\citep[see][]{Ko2015, Chen2016, Okkalioglu2023, Alshehri2023}. Our choice of the gamma distribution was primarily based on its support on the positive real line, but one can explore other distributions such as the log-normal distribution or the inverse-gamma distribution. This research has the potential to further ground the foundations of text analysis in statistical testing while simultaneously spurring the development of new and improved term-weighting schemes.

\begin{acks}[Acknowledgements]
The authors would like to thank the two anonymous referees and the managing editor for their constructive comments that improved the quality of this paper.
\end{acks}

\begin{acks}[Code availability]
Code used to produce the numerical examples in this work is available at the GitHub repository \url{https://github.com/sheridan-stable/tfidf-lrt}, Release v2026.04.
\end{acks}

\begin{acks}[Contributions]
\textbf{Zeyad Ahmed}: Conceptualization, Formal analysis, Investigation, Methodology, Numerical experiments, Proofs, Results interpretation, Software, Visualization, Writing – original draft, Writing – review \& editing. \textbf{Paul Sheridan}: Conceptualization, Methodology, Project administration, Proofs, Results interpretation, Supervision, Validation, Writing – original draft, Writing – and review \& editing. \textbf{Michael McIsaac}: Funding acquisition, Methodology, Results interpretation, Supervision, Writing – review \& editing. \textbf{Aitazaz A. Farooque}: Funding acquisition, Supervision, Results interpretation, Writing – review \& editing. All authors reviewed the results and approved the final version of the manuscript.
\end{acks}

\begin{funding}
The work was supported by the Natural Science and Engineering Research Council of Canada [RGPIN-2016-04384 to M.M. and NSERC Discovery and NSERC Alliance Sustainable Agriculture Research Initiative Grants to A.A.F.].
\end{funding}

\bibliographystyle{imsart-number}
\bibliography{bibliography}

\begin{appendix}
\numberwithin{equation}{section} 
\renewcommand{\thefigure}{\Alph{section}\arabic{figure}} 
\setcounter{figure}{0} 

\section{Technical details} \label{appn}

\subsection{First-order approximation of the gamma function} \label{appn:gamma-approx}
We derive the approximation
\begin{equation*}
\Gamma(x + a) = \Gamma(x) + \bigO(a)
\end{equation*}
for small values of $a$, using a Taylor expansion of the logarithm of the gamma function.

Recall that the digamma function is defined as the logarithmic derivative:
\begin{equation*}
\psi(x) = \frac{d}{dx} \log \Gamma(x).
\end{equation*}
Expanding $\log \Gamma(x + a)$ around $a = 0$ yields:
\begin{eqnarray*}
    \log \Gamma(x + a) &=& \log \Gamma(x) + a \psi(x) + \sum_{k=2}^\infty \frac{a^k}{k!} \psi^{(k-1)}(x) \\
    &=& \log \Gamma(x) + a \psi(x) + \bigO(a^2).
\end{eqnarray*}
Exponentiating both sides gives:
\begin{equation*}
\Gamma(x + a) = \Gamma(x) \cdot \exp\left(a \psi(x) + \bigO(a^2)\right).
\end{equation*}
Expanding the exponential for small $a$:
\begin{equation*}
\exp\left(a \psi(x) + \bigO(a^2)\right) = 1 + a \psi(x) + \bigO(a^2).
\end{equation*}
So we have:
\begin{eqnarray*}
    \Gamma(x + a) &=& \Gamma(x) \left(1 + a \psi(x) + \bigO(a^2) \right) \\
    &=& \Gamma(x) + a \Gamma(x)\psi(x) + \bigO(a^2).
\end{eqnarray*}
Hence, for small $a$, it follows that:
\begin{equation*}
\Gamma(x + a) = \Gamma(x) + \bigO(a).
\end{equation*}

\subsection{Evaluating the logarithm of the gamma penalty term} \label{appn:log-gamma}
Starting from the gamma penalty term in Eq.~\eqref{eq:gamma-penalty},
\begin{equation*}
    f(\hat{\alpha}_{0i}; \mu, \sigma^2) 
    = \frac{\left( \frac{\mu}{\sigma^2} \right)^{\mu^2 / \sigma^2}}
    {\Gamma\!\left( \frac{\mu^2}{\sigma^2} \right)}
    \hat{\alpha}_{0i}^{\frac{\mu^2}{\sigma^2} - 1}
    \exp\!\left( -\frac{\mu}{\sigma^2} \hat{\alpha}_{0i} \right),
\end{equation*}
with the substitution
\begin{equation*}
    \hat{\alpha}_{0i} = \frac{\sigma^2}{\mu} \left( b_i - n_i - 1 + \frac{\mu^2}{\sigma^2} \right),
\end{equation*}
and defining
\begin{equation*}
    \eta = \frac{\mu}{\sigma}, \quad r_i = n_i - b_i + 1,
\end{equation*}
we can write
\begin{equation*}
    \hat{\alpha}_{0i} = \frac{\sigma^2}{\mu} \left( \eta^2 - r_i \right).
\end{equation*}
Taking logs:
\begin{equation*}
    \log f(\hat{\alpha}_{0i}; \mu, \sigma^2) = \eta^2 \log\frac{\mu}{\sigma^2} + (\eta^2 - 1)\log \hat{\alpha}_{0i} - \frac{\mu}{\sigma^2} \hat{\alpha}_{0i} - \log \Gamma(\eta^2).
\end{equation*}
Substitute $\hat{\alpha}_{0i}$:
\begin{equation*}
    \eta^2 \log\frac{\mu}{\sigma^2}
    + (\eta^2 - 1)\log\left( \frac{\sigma^2}{\mu} (\eta^2 - r_i) \right) - (\eta^2 - r_i) - \log \Gamma(\eta^2).
\end{equation*}
Combining the first two log terms:
\begin{eqnarray*}
    && \eta^2 \log\frac{\mu}{\sigma^2}
    + (\eta^2 - 1)\left[ \log(\eta^2 - r_i) + \log\frac{\sigma^2}{\mu} \right] \nonumber \\
    &=& (\eta^2 - 1)\log(\eta^2 - r_i) + \log\frac{\mu}{\sigma^2}.
\end{eqnarray*}
So:
\begin{equation*}
    \log f(\hat{\alpha}_{0i}; \mu, \sigma^2) = (\eta^2 - 1)\log(\eta^2 - r_i) + \log\frac{\mu}{\sigma^2} - (\eta^2 - r_i) - \log \Gamma(\eta^2).
\end{equation*}
Using the Stirling expansion~\cite[\href{https://dlmf.nist.gov/5.11\#E1}{(5.11.1)}]{NIST2025}:
\begin{equation*}
    \log \Gamma(z) \approx \left(z - \frac12\right)\log z - z + \log\sqrt{2\pi},
\end{equation*}
with $z = \eta^2$ we get
\begin{equation*}
    -\log\Gamma(\eta^2) \approx -\left(\eta^2 - \frac12\right)\log\eta^2 + \eta^2 - \log\sqrt{2\pi},
\end{equation*}
giving
\begin{equation*}
    \begin{split}
    \log f(\hat{\alpha}_{0i}; \mu, \sigma^2) &\approx (\eta^2 - 1)\log(\eta^2 - r_i) + \log\frac{\mu}{\sigma^2} - (\eta^2 - r_i) - \left(\eta^2 - 1/2\right)\log\eta^2 + \eta^2 \\
    &\quad - \log\sqrt{2\pi}.
    \end{split}
\end{equation*}
Writing $\log\frac{\mu}{\sigma^2} = \log\eta^2 - \log\mu$ and simplifying
\begin{eqnarray*}
    \log f(\hat{\alpha}_{0i}; \mu, \sigma^2) &\approx& (\eta^2 - 1)\log(\eta^2 - r_i) + \log\eta^2 - \log\mu + r_i - \left(\eta^2 - 1/2\right)\log\eta^2 \\
    &\quad& - \log\sqrt{2\pi} \nonumber \\
    &=& (\eta^2 - 1)\log(\eta^2 - r_i) - \log\mu + r_i - \left(\eta^2 - 3/2\right)\log\eta^2 - \log\sqrt{2\pi}.
\end{eqnarray*}

\subsection{Sensitivity analyses}\label{appn:sensitivity-analyses}
In this appendix, we visualize the overall accuracy scores for multinomial Naive Bayes classification tasks using $\mathcal{S}(\lambda_{ij})$ features on the datasets used in this paper, namely \textit{20~Newsgroups} and \textit{R8}, across different hyperparameters configurations. We evaluate the sensitivity of classification accuracy to the penalty term parameters, specifically $\mu \in \{30, 60, 90, 120, 150, 180\}$ and $\sigma^2 \in \{0.5, 1.0, 2.0, 3.0, 4.0, 5.0\}$. For both datasets, the text was preprocessed by converting to lowercase, removing punctuation, non-alphabetic characters, and stopwords, and collapsing multiple whitespaces.

As described in \Cref{subsec:numerical-real}, the resulting vocabulary for the \textit{20~Newsgroups} dataset comprises 88,701 unique terms, with a mean training document length of 165.29. Figure~\ref{fig:20ng-sensitivity-analysis} illustrates the accuracy scores across the $(\mu, \sigma^2)$ grid. The results indicate that the classifier is stable, with accuracy minimally affected by the choice of the hyperparameters. In this analysis, the optimal pair ($\mu=60, \sigma^2=2.0$) achieved an accuracy of 0.80868, outperforming other pairs by only a marginal difference.

\begin{figure}[!ht]
    \centering
    \includegraphics[width=0.9\linewidth]{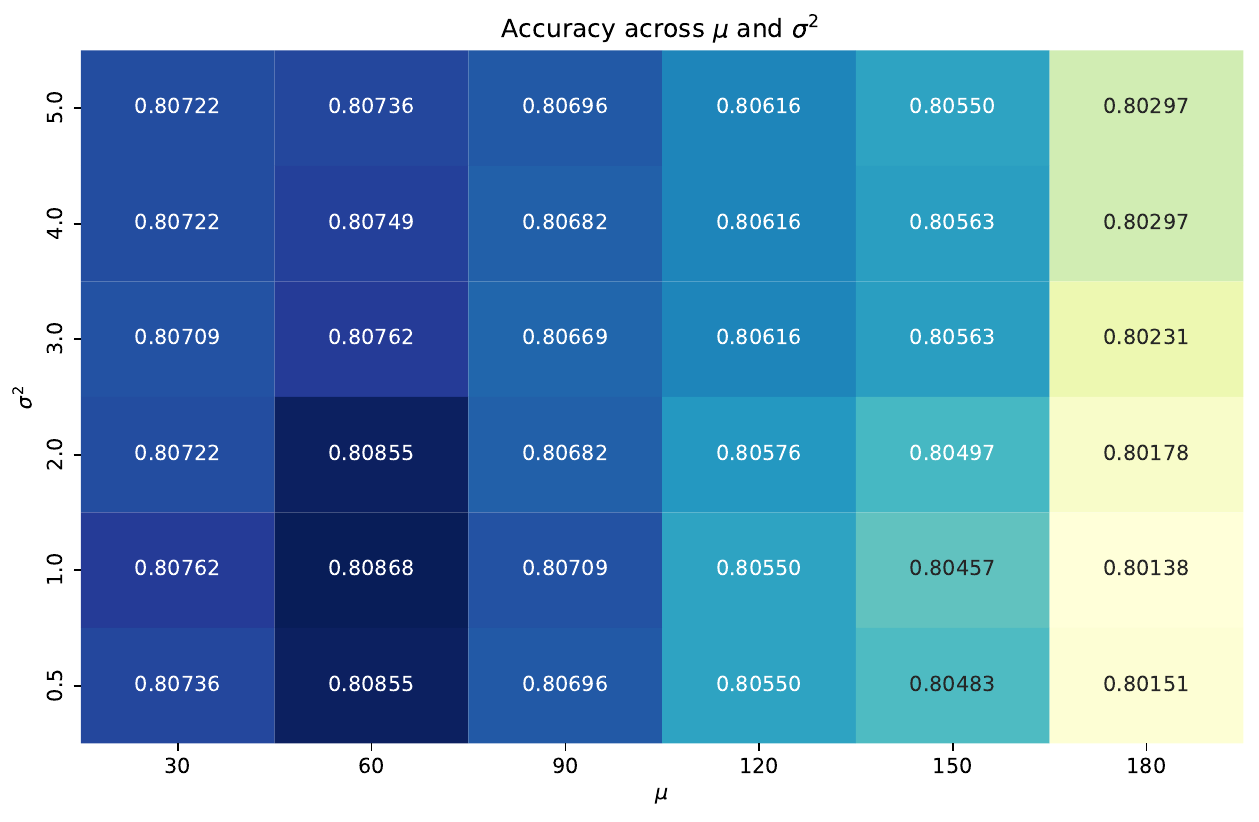}
    \caption{Sensitivity of classification accuracy for a Naive Bayes classifier using $\mathcal{S}(\lambda_{ij})$ features, with respect to the hyperparameters $\mu\in\{30, 60, 90, 120, 150, 180\}$ and $\sigma^2\in\{0.5, 1.0, 2.0, 3.0, 4.0, 5.0\}$, evaluated on the \textit{20~Newsgroups} dataset.}
    \label{fig:20ng-sensitivity-analysis}
\end{figure}

The \textit{R8} dataset consists of 19,688 unique terms and a mean training document length of 65.75 as described in \Cref{subsec:numerical-real-ii}. Figure~\ref{fig:r8-sensitivity-analysis} presents the accuracy grid for the same hyperparameter grid. Consistent with the \textit{20~Newsgroups} results, the performance remains stable across all values. The peak accuracy of 0.94521 was achieved at $(\mu=120, \sigma^2=1.0)$ and $(\mu=180, \sigma^2 \in \{1.0, 2.0\})$.

\begin{figure}[!ht]
    \centering
    \includegraphics[width=0.9\linewidth]{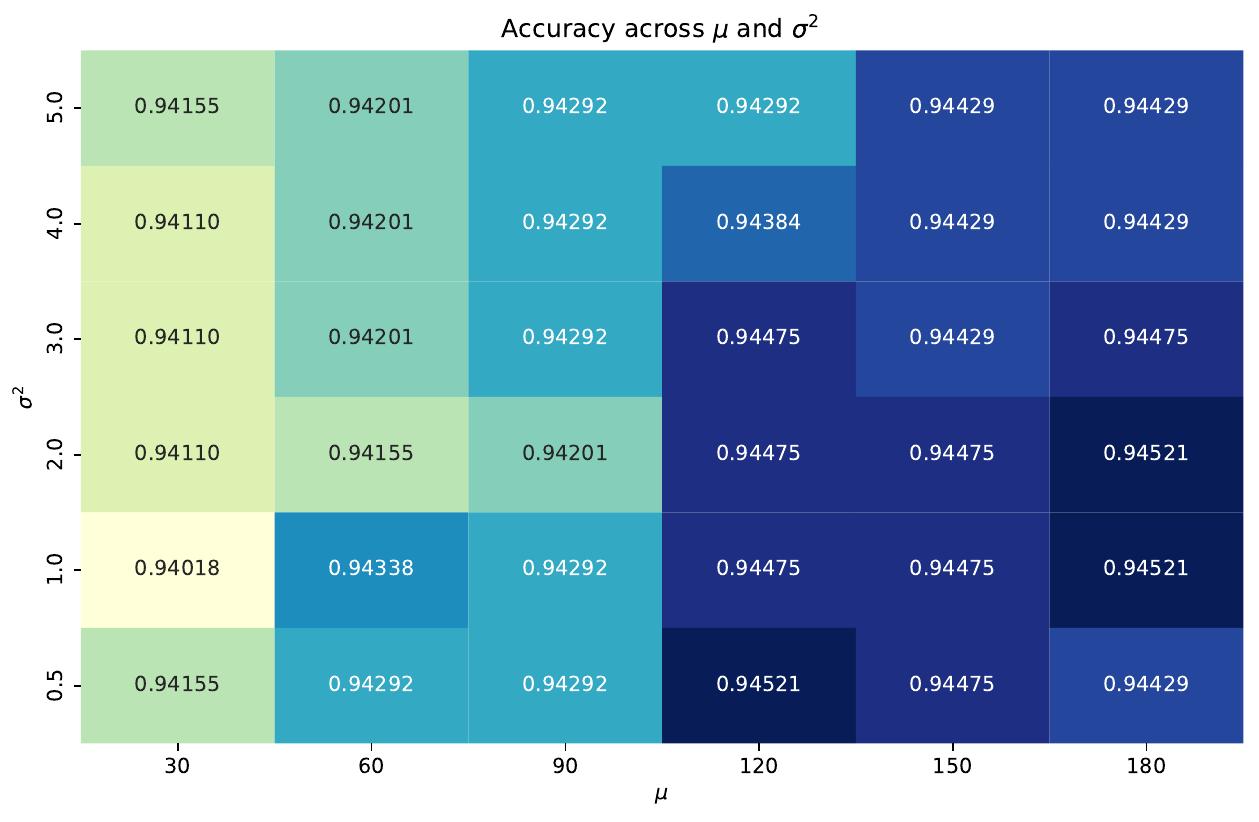}
    \caption{Sensitivity of classification accuracy for a Naive Bayes classifier using $\mathcal{S}(\lambda_{ij})$ features, with respect to the hyperparameters $\mu\in\{30, 60, 90, 120, 150, 180\}$ and $\sigma^2\in\{0.5, 1.0, 2.0, 3.0, 4.0, 5.0\}$, evaluated on the \textit{R8} dataset.}
    \label{fig:r8-sensitivity-analysis}
\end{figure}

Notably, for $\sigma^2 > 1.0$, the term $(\eta^2 - 2r_i + 1) \log(\max(1, \eta^2 - r_i))$ in Eq.~\eqref{eq:lambda-ij} often evaluates to zero, as $\eta^2 - r_i < 1$ for most terms as discussed in \Cref{sec:gbb-mles,subsec:evaluation-of-lrt,subsec:numerical-real}.
\end{appendix}

\end{document}